%% file: paper.tex
\documentclass{article}

\usepackage[final,nonatbib]{neurips_2023}

\usepackage{multirow}
\usepackage[utf8]{inputenc} 
\usepackage[T1]{fontenc}    
\usepackage{url}            
\usepackage{booktabs}       
\usepackage{amsfonts}       
\usepackage{nicefrac}       
\usepackage{microtype}      
\usepackage{xcolor}         
\usepackage{float}
\usepackage{amsmath}
\usepackage{amsthm}
\usepackage{latexsym}
\usepackage{amssymb}
\usepackage{etoolbox}
\usepackage{array}
\usepackage{booktabs}
\usepackage{enumitem}
\usepackage{titlesec}
\usepackage{caption}
\usepackage{subcaption}
\usepackage{comment}
\usepackage{algorithm}
\usepackage{algorithmicx}
\usepackage[noend]{algpseudocode}
\usepackage{pifont}
\usepackage{verbatim,ifthen,alltt}
\usepackage{relsize}
\usepackage{microtype}      
\usepackage{stmaryrd}
\usepackage{xfrac}
\usepackage{tikz}
\usepackage{pgf}

\usetikzlibrary{fit,calc}

\input{cdefs}

\def\pdfauthor{X. Huang and J. Marques-Silva}
\usepackage[pdftex%
,colorlinks=true%
,bookmarks=true%
,linkcolor=darkblue%
,citecolor=darkblue%
,urlcolor=darkblue%
,plainpages=false]{hyperref}
\hypersetup{%
pdfauthor={{\pdfauthor}},
pdftitle={{}}
}
\usepackage[noabbrev,nameinlink,capitalise]{cleveref}

\input{defs}

\input{preamble}

\begin{document}

\maketitle

\input{abs}

\input{intro}

\input{prelim}

\input{gxp}

\input{rob2xp}

\input{res}
\input{conc}
\input{acks}

%

\input{replbib}
\input{togbbl} 

\addcontentsline{toc}{section}{References}
\vskip 0.2in
\iftoggle{mkbbl}{
  \bibliographystyle{abbrv}
  \bibliography{refs,xrefs,xtra}
}{

\input{paper.bibl}
}
\label{lastpage}

\clearpage
\appendix
\input{appendix}
\end{document}

%% file: cdefs.tex
\definecolor{darkslategray}{rgb}{0.18, 0.31, 0.31} 
\definecolor{platinum}{rgb}{0.9, 0.89, 0.89} 
\definecolor{gray}{rgb}{.4,.4,.4}
\definecolor{midgrey}{rgb}{0.5,0.5,0.5}
\definecolor{middarkgrey}{rgb}{0.35,0.35,0.35}
\definecolor{darkgrey}{rgb}{0.3,0.3,0.3}
\definecolor{darkred}{rgb}{0.7,0.1,0.1}
\definecolor{midblue}{rgb}{0.2,0.2,0.7}
\definecolor{darkblue}{rgb}{0.1,0.1,0.5}
\definecolor{darkgreen}{rgb}{0.1,0.5,0.1}
\definecolor{defseagreen}{cmyk}{0.69,0,0.50,0}

%% file: defs.tex
\newtheorem{proposition}{Proposition}
\newtheorem{definition}{Definition}

\newtheorem{example}{Example}
\newtheorem{remark}{Remark}

\crefname{theorem}{Theorem}{Theorems}
\crefname{lemma}{Lemma}{Lemmas}
\crefname{proposition}{Proposition}{Propositions}
\crefname{definition}{Definition}{Definitions}
\crefname{corollary}{Corollary}{Corollaries}
\crefname{example}{Example}{Examples}
\crefname{claim}{Claim}{Claims}
\crefname{assumption}{Assumption}{Assumptions}
\crefname{enumi}{}{}

\newcommand{\fml}[1]{{\mathcal{#1}}}

\newcommand{\tn}[1]{\textnormal{#1}}
\newcommand{\tbf}[1]{\textbf{#1}}
\newcommand{\mbf}[1]{\ensuremath\mathbf{#1}}
\newcommand{\msf}[1]{\ensuremath\mathsf{#1}}
\newcommand{\mbb}[1]{\ensuremath\mathbb{#1}}

\newcommand{\waxp}{\ensuremath\mathsf{WAXp}}
\newcommand{\wcxp}{\ensuremath\mathsf{WCXp}}
\newcommand{\axp}{\ensuremath\mathsf{AXp}}
\newcommand{\cxp}{\ensuremath\mathsf{CXp}}
\newcommand{\aex}{\ensuremath\mathsf{AEx}}

\newcommand{\findaxpdel}{\ensuremath\msf{FindAXpDel}}
\newcommand{\findaxpqxp}{\ensuremath\msf{FindAXpQxp}}
\newcommand{\qxprecur}{\ensuremath\msf{QxpRecur}}
\newcommand{\findcxpdel}{\ensuremath\msf{FindCXpDel}}
\newcommand{\robt}{\ensuremath\msf{FindAdvEx}}

\newcommand{\outc}{\oper{hasAE}}

\newcommand{\False}{\textbf{false}}

\newcommand{\pnorm}[1]{\ensuremath{l}_{#1}}
\newcommand{\refd}{\ensuremath\epsilon}
\newcommand{\refdg}{\ensuremath\eta}

\newcommand{\oper}[1]{\ensuremath\mathsf{#1}}

\newcounter{tableeqn}[table]

\DeclareMathOperator*{\limply}{\rightarrow}

\newcommand{\jnoteF}[1]{}

\newcolumntype{L}[1]{>{\raggedright\let\newline\\\arraybackslash\hspace{0pt}}m{#1}}
\newcolumntype{C}[1]{>{\centering\let\newline\\\arraybackslash\hspace{0pt}}m{#1}}
\newcolumntype{R}[1]{>{\raggedleft\let\newline\\\arraybackslash\hspace{0pt}}m{#1}}

\tikzset{
  0 my edge/.style={densely dashed, my edge},
  my edge/.style={-{Stealth[]}},
}

\def\HiLi{\leavevmode\rlap{\hbox to \linewidth{\color{platinum}\leaders\hrule height .8\baselineskip depth .5ex\hfill}}}

\setitemize{noitemsep,topsep=0pt,parsep=0pt,partopsep=0pt}
\setenumerate{noitemsep,topsep=0pt,parsep=0pt,partopsep=0pt}
\setlist{nosep,leftmargin=0.45cm}

\providetoggle{long}
\settoggle{long}{false}

\makeatletter
\algnewcommand{\LineComment}[1]{\Statex \hskip\ALG@thistlm \(\triangleright\) #1}
\makeatother

%% file: preamble.tex
\title{%
  From Robustness to Explainability and Back Again
}

\author{%
  Xuanxiang Huang
  \\
  University of Toulouse \\
  Toulouse, France \\
  \texttt{xuanxiang.huang@univ-toulouse.fr} \\
  \And
  Joao Marques-Silva \\
  IRIT, CNRS, France \\
  Toulouse, France\\
  \texttt{joao.marques-silva@irit.fr} \\
}

%% file: abs.tex
\begin{abstract}
  Formal explainability guarantees the rigor of computed explanations,
  and so it is paramount in domains where rigor is critical, including
  those deemed high-risk.
  Unfortunately, since its inception formal explainability has been
  hampered by poor scalability. At present, this limitation still
  holds true for some families of classifiers, the most significant
  being deep neural networks.
  This paper addresses the poor scalability of formal explainability
  and proposes novel efficient algorithms for computing formal
  explanations.
  The novel algorithm computes explanations by answering instead a
  number of robustness queries, and such that the number of such
  queries is at most linear on the number of features.
  Consequently, the proposed algorithm establishes a direct
  relationship between the practical complexity of formal
  explainability and that of robustness.
  To achieve the proposed goals, the paper generalizes the definition
  of formal explanations, thereby allowing the use of robustness tools
  that are based on different distance norms, and also by reasoning in
  terms of some target degree of robustness.
  Preliminary experiments validate the practical efficiency of the
  proposed approach.
\end{abstract}

%% file: intro.tex
\section{Introduction} \label{sec:intro}

The importance of rigorous analysis of complex ML models is widely
acknowledged~\cite{seshia-cacm22,dong-corr22}, and critical for the
deployment of systems of artificial intelligence (AI) in
high-risk and safety-critical domains. Hence, rigorous analysis of
complex ML models is paramount when analyzing their robustness,
fairness and explainability.

Recent years witnessed a large body work on the assessment of
robustness,
e.g.~\cite{pulina-cav10,barrett-cav17,kwiatkowska-cav17,carlini-sp17,ehlers-atva17,madry-iclr18,carlini-icml18,kwiatkowska-tacas18,kwiatkowska-ijcai18,vechev-sp18,ermon-nips18,vechev-iclr19,kumar-iclr19,kolter-icml19,barrett-cav19,kwiatkowska-cvpr19,kwiatkowska-ijcai19,kohli-iccv19,kohli-emnlp19,vechev-nips19a,vechev-nips19b,kohli-nips19a,kohli-nips19b,kwiatkowska-tcs20,levine-aaai20,kolter-icml20,hein-icml20,vechev-icml20,kwiatkowska-cvpr20,kumar-jmlr20,kumar-iclr20,kumar-uai20,barrett-fto21,vechev-aaai21,vechev-pldi21,kwiatkowska-ijcai21,vechev-popl22,vechev-iclr22a,vechev-iclr22b,kolter-icml22,tramer-icml22},
with significant progress observed. Analysis of robustness is
motivated by the existence of adversarial examples in complex ML
models~\cite{szegedy-iclr14}.
More recently, there have been massive efforts regarding the
explainability of ML models. Unfortunately, most explainability 
approaches offer no guarantees of
rigor~\cite{muller-plosone15,guestrin-kdd16,lundberg-nips17,guestrin-aaai18,muller-ieee-proc21}.
Indeed, there is practical evidence that explanations computed
with well-known non-formal methods can produce incorrect
results~\cite{ignatiev-ijcai20}.
A recent alternative are formal approaches to
explainability~\cite{darwiche-ijcai18,msi-aaai22}.
%
These approaches propose logic-based definitions of
explanations, which are model-precise, and so ensure absolute
guarantees of rigor.

\jnoteF{ToDo.}

\jnoteF{Vast body of work on robustness, with many approaches that are
  both sound and complete. A number of works targets deep ML modes,
  including NNs.}

\jnoteF{Most XAI solutions offer no guarantees of rigor.\\
  Formal XAI offers guarantees of rigor, but there is a perception of
  lack of scalability.}

\jnoteF{Another downside of formal XAI is that dedicated solutions need
  to be devised for each family of classifiers.}

\jnoteF{Furthermore, it is widely accepted that XAI for NNs does not
  scale beyond a few tens of activation
  units~\cite{ignatiev-ijcai20,msi-aaai22}.
}

Despite the observed progress, formal XAI still faces critical
challenges. Among these, the most significant is arguably the
inability of these methods for explaining (deep) neural networks
((D)NNs), but also a host of other ML classifiers, for which logic
encodings have not yet been devised (e.g.\ support vector machines
(SVM's) and $k$ nearest neighbors classifiers among several others). 
Unsurprisingly, the inability of formal XAI methods to scale for
complex DNNs has been acknowledged over the
years~\cite{inms-aaai19,ignatiev-ijcai20,msi-aaai22,barrett-corr22,katz-tacas23}. Existing
data~\cite{inms-aaai19} indicates that formal explanations can be
computed for DNNs containing a few tens of ReLU units.
Furthermore, the analysis of similar-sized NNs with other non-ReLU
activation units would require dedicated logic encodings, which at
present are yet to be devised.

Another challenge of formal XAI is explanation
size~\cite{msi-aaai22}, given the well-known limitations of human
decision makers to reason with a large number of
concepts~\cite{miller-pr56}.
Recent work addressed this challenge by proposing probabilistic
explanations~\cite{kutyniok-jair21,barcelo-nips22}.
%
However, given the complexity of computing such probabilistic
explanations, it seems unlikely that practical methods will be devised
that will scale in the case of DNNs.

\jnoteF{There have been efforts towards relating explainability and
  robustness, but the focus was on global explanations, and not on
  instance-specific explanations.}

Moreover, the relationship between explanations and robustness remains
elusive, despite some initial results~\cite{inms-nips19}. Hoever, such
results relate globally defined explanations and so-called
counterexamples, and are not readily applicable in the case of
(locally-defined) abductive and contrastive explanations.
One possible justification is that adversarial examples are defined 
with respect to some target distance $\refd$ of some given norm $l_p$,
and this is orthogonal to the definitions of formal explanations,
which solely consider sets of features.
As a result, the observed progress in assessing robustness has not
been exploited in explainability.

\jnoteF{This paper addresses the problem of relating explainability and
  robustness, but approaches the problem from a very different angle.
  One crucial aspect is that we take into account the different norms
  that are used in robustness.}

\jnoteF{As a result, the paper proposes a generalization of the best
  known types of formal explanations, namely abductive and contrastive
  explanations. 
  The purpose of the generalization is to account for the possible
  norms used in robustness.
  The proposed generalization is such that each contrastive explanation
  captures exactly one (or a family of) adversarial example(s).
}

In contrast with past work on formal XAI, this paper reveals a
rigorous relationship between adversarial examples and formal 
explanations. This involves generalizing the definitions of formal
explanations to account for a value of distance $\refd$ for a given
norm $\pnorm{p}$.
More importantly, the generalized definition of formal explanation
proposed in this paper allows exploiting a wide range of
robustness tools (and in principle any such robustness tool) for the
purpose of computing explanations.
The immediate consequence of using robustness tools for computing
explanations, is that this makes explainability directly dependent on
the performance of robustness tools. And these tools have been the
subject of impressive performance gains in recent years.
The paper also proves a duality result between explanations that take
into account the value $\refd$ of distance, thereby enabling the use
of other formal explainability algorithms, e.g.\ enumeration of
explanations.
One additional result is that the paper shows that the number of
explanations is directly related with the distance that is considered
when assessing robustness.

\jnoteF{Furthermore, the paper proves a crucial minimal hitting set
  duality relationship between the generalized definitions of
  abductive and contrastive explanations.}

\jnoteF{More importantly, the paper then exploits MHS duality to devise
  algorithms for computing one abductive explanation, which build on
  existing tools for finding adversarial examples, and such that only
  mild requirements need to be met by those tools.}

%% file: prelim.tex
\section{Preliminaries} \label{sec:prelim}

\paragraph{Minimal hitting sets.}
%
Let $\fml{S}$ be a set and $\mbb{B}\subseteq2^{\fml{S}}$ be a set of
subsets of $\fml{S}$. A hitting set (HS) $\fml{H}\subseteq\fml{S}$ of
$\mbb{B}$ is such that
$\forall(\fml{P}\in\mbb{B}).\fml{P}\cap\fml{H}\not=\emptyset$. A
minimal hitting set (MHS) $\fml{Q}\subseteq\fml{S}$ is a hitting set
of $\mbb{B}$ such that no proper subset of $\fml{Q}$ is a hitting
set of $\mbb{B}$,
i.e.\ $\forall(\fml{R}\subsetneq\fml{Q})\exists(\fml{P}\in\mbb{B}).\fml{R}\cap\fml{P}=\emptyset$.
A minimal hitting set is said to be subset-minimal or irreducible.

\paragraph{Norm $\pnorm{p}$.}
%
The distance between two vectors $\mbf{v}$ and $\mbf{u}$ is denoted by
$||\mbf{v}-\mbf{u}||$, and the actual definition depends on the norm
being considered.
Different norms $\pnorm{p}$ can be considered.
For $p\ge1$, the $p$-norm is defined as follows~\cite{horn-bk12}:
\begin{equation}
  \begin{array}{lcl}
    ||\mbf{x}||_{p} & {:=} &
    \left(\sum\nolimits_{i=1}^{m}|x_i|^{p}\right)^{\sfrac{1}{p}}
  \end{array}
\end{equation}

Let $d_i=1$ if $x_i\not=0$, and let $d_i=0$ otherwise. Then, for
$p=0$, we define the 0-norm, $\pnorm{0}$, as
follows~\cite{robinson-bk03}:
\begin{equation}
  \begin{array}{lcl}
    ||\mbf{x}||_{0} & {:=} &
    \sum\nolimits_{i=1}^{m}d_i
  \end{array}
\end{equation}

In general, for $p\ge1$, $\pnorm{p}$ denotes the Minkowski distance.
Well-known special cases include
the Manhattan distance $\pnorm{1}$,
the Euclidean distance $\pnorm{2}$, and
the Chebyshev distance $\pnorm{\infty}$.
$\pnorm{0}$ denotes the Hamming distance.

\paragraph{Classification problems.}
%
Classification problems are defined on a set of features
$\fml{F}=\{1,\ldots,m\}$ and a set of classes
$\fml{K}=\{c_1,\ldots,c_K\}$.
Each feature $i$ has a domain $\mbb{D}_i$. Features can be ordinal or
categorical. Ordinal features can be discrete or real-valued.
Feature space is defined by the cartesian product of the features'
domains: $\mbb{F}=\mbb{D}_1\times\cdots\times\mbb{D}_m$.
A classifier computes a total function $\kappa:\mbb{F}\to\fml{K}$.
Throughout the paper, a classification problem $\fml{M}$ represents a
tuple $\fml{M}=(\fml{F},\mbb{F},\fml{K},\kappa)$.

An instance (or a sample) is a pair $(\mbf{v},c)$, with
$\mbf{v}\in\mbb{F}$ and $c\in\fml{K}$.
An explanation problem $\fml{E}$ is a tuple
$\fml{E}=(\fml{M},(\mbf{v},c))$. The generic purpose of XAI is to find
explanations for each given instance.
Moreover, when reasoning in terms of robustness, we are also
interested in the behavior of a classifier given some instance.
Hence, we will also use explanation problems also in the case of
robustness.

\begin{example} \label{ex:runex}
  Throughout the paper, we consider the following classification
  problem.
  The features are $\fml{F}=\{1,2,3\}$, all ordinal with domains
  $\mbb{D}_1=\mbb{D}_2=\mbb{D}_3=\mbb{R}$. The set of classes is
  $\fml{K}=\{0,1\}$. Finally, the classification function is
  $\kappa:\mbb{F}\to\fml{K}$, defined as follows (with
  $\mbf{x}=(x_1,x_2,x_3)$): 
  \[
  \kappa(\mbf{x})=\left\{
  \begin{array}{lcl}
    1 & ~~ & \tn{if~} 0<{x_1}<2 \land 4x_1\ge(x_2+x_3) \\[3pt]
    0 & ~~ & \tn{otherwise}
  \end{array}
  \right.
  \]
  Moreover, let the target instance be $(\mbf{v},c)=((1,1,1),1)$.
\end{example}

\paragraph{Robustness.}
%
Let $\fml{M}=(\fml{F},\mbb{F},\fml{K},\kappa)$ be a classification
problem. Let $(\mbf{v},c)$, with $\mbf{v}\in\mbb{F}$ and
$c=\kappa(\mbf{v})$, be a given instance. Finally, let $\refd>0$ be
a value distance for norm $l_p$.

We say that there exists an adversarial example if the following logic
statement holds true,
\begin{equation} \label{eq:locrob}
  \exists(\mbf{x}\in\mbb{F}).\left(||\mbf{x}-\mbf{v}||_p\le\refd\right)\land\left(\kappa(\mbf{x})\not=c)\right) 
\end{equation}
(The logic statement above holds true if there exists a point
$\mbf{x}$ which is less than $\refd$ distance (using norm $\pnorm{p}$)
from $\mbf{v}$, and such that the prediction changes.)
If~\eqref{eq:locrob} is false, then the classifier is said to be
$\refd$-robust.
If~\eqref{eq:locrob} is true, then any $\mbf{x}\in\mbb{F}$ for which
the following predicate holds\footnote{Parameterizations are shown as
predicate arguments positioned after ';'. These may be dropped for the
sake of brevity.}:
%
\begin{equation} \label{eq:ae}
  \aex(\mbf{x};\fml{E},\refd,p) ~:=~
  \left(||\mbf{x}-\mbf{v}||_p\le\refd\right)\land\left(\kappa(\mbf{x})\not=c)\right) 
\end{equation}
is referred to as an \emph{adversarial example}.

\begin{example} \label{ex:runex:ae}
  For the classifier from~\cref{ex:runex}, for distance $l_1$, and
  with $\refd=1$, there exist adversarial examples by either setting
  $x_1=0$ or $x_1=2$.
\end{example}
    
\paragraph{Logic-based explanations.}
%
In the context of explaining ML models, rigorous, model-based,
explanations have been studied since 2018~\cite{darwiche-ijcai18}. We
follow recent treatments of the
subject~\cite{msi-aaai22,marquis-dke22,darwiche-jlli23}.
%
A PI-explanation (which is also referred to as an abductive
explanation (AXp)~\cite{msi-aaai22}) is a irreducible subset of the
features which, if fixed to the values specified by an instance, are
sufficient for the prediction.

Given an instance $(\mbf{v},c)$, a set of features
$\fml{X}\subseteq\fml{F}$ is sufficient for the prediction if the
following logic statement holds true,
\begin{equation} \label{eq:waxp}
  \forall(\mbf{x}\in\mbb{F}).\left[\land_{i\in\fml{X}}(x_i=v_i)\right]\limply\left(\kappa(\mbf{x})=c)\right) 
\end{equation}
If~\eqref{eq:waxp} holds, but $\fml{X}$ is not necessarily irreducible
(i.e.\ it is not subset minimal), then we say that $\fml{X}$ is a weak
abductive explanation (WAXp). As a result, we associate a predicate
$\waxp$ with~\eqref{eq:waxp}, such that $\waxp(\fml{X};\fml{E})$ holds
true iff~\eqref{eq:waxp} holds true. An AXp is a weak AXp that is
subset-minimal. The predicate $\axp(\fml{X};\refd,\fml{E})$ holds true iff
set $\fml{X}$ is also an AXp. An AXp answers a \emph{Why?} question,
i.e.\ why is the prediction $c$ (given the existing features).

A set of features $\fml{Y}\subseteq\fml{F}$ is sufficient for changing
the prediction if the following logic statement holds true,
\begin{equation} \label{eq:wcxp}
  \exists(\mbf{x}\in\mbb{F}).\left[\land_{i\in\fml{F}\setminus\fml{Y}}(x_i=v_i)\right]\land\left(\kappa(\mbf{x})\not=c)\right) 
\end{equation}
If~\eqref{eq:wcxp} holds, but $\fml{Y}$ is not necessarily
irreducible, then we say that $\fml{Y}$ is a weak contrastive
explanation (CXp). As a result, we associate a predicate $\wcxp$
with~\eqref{eq:wcxp}, such that $\wcxp(\fml{Y};\fml{E})$ holds true 
iff~\eqref{eq:wcxp} holds true. A CXp is a weak CXp that is also
subset-minimal. The predicate $\cxp(\fml{Y};\fml{E})$ holds true iff
set $\fml{Y}$ is a CXp. It is well-known that a CXp answers a
\emph{Why~Not?} question~\cite{miller-aij19,inams-aiia20}.

\begin{example}
  For the classifier of~\cref{ex:runex}, and given the target instance,
  the AXp is $\{1,2,3\}$. Clearly, if we allow any feature to take any
  value, then we can change the prediction. Hence, the prediction does
  not change only if all features are fixed.
\end{example}

Given the above, we define,
\begin{align}
  \mbb{A}(\fml{E}) = \{\fml{X}\subseteq\fml{F}\,|\,\axp(\fml{X};\fml{E})\}
  \label{eq:allaxp}\\
  \mbb{C}(\fml{E}) = \{\fml{X}\subseteq\fml{F}\,|\,\cxp(\fml{X};\fml{E})\}
  \label{eq:allcxp} 
\end{align}
which capture, respectively, the set of all AXp's and the set of all
CXp's given an explanation problem $\fml{E}$.

Finally, the following result relating AXp's and CXp's is used
extensively in devising explainability
algorithms~\cite{inams-aiia20}.
\begin{proposition}[MHS Duality between AXp's and CXp's] \label{prop:duality1}
  Given an explanation problem $\fml{E}$, and norm $p$ and a value
  $\refd>0$ then,
  \begin{enumerate}[nosep]
  \item A set $\fml{X}\subseteq\fml{F}$ is an AXp iff $\fml{X}$ a
    minimal hitting set of the CXp's in $\mbb{C}(\fml{E})$.
  \item A set $\fml{X}\subseteq\fml{F}$ is a CXp iff $\fml{X}$ a
    minimal hitting set of the AXp's in $\mbb{A}(\fml{E})$.
  \end{enumerate}
\end{proposition}
(\cref{prop:duality1} is a consequence of an earlier seminal result in
model-based diagnosis~\cite{reiter-aij87}.)
\cref{prop:duality1} is instrumental for enumerating abductive (but
also contrastive) explanations~\cite{inams-aiia20}.
In contrast with non-formal explainability, the navigation of the
space of abductive (or contrastive) explanations, i.e.\ their
enumeration, is a hallmark of formal
XAI~\cite{marquis-kr21,msi-aaai22}.) 

\jnoteF{To
  cite:~\cite{darwiche-jlli23,marquis-dke22,msi-aaai22,ms-corr22}.}




Progress in formal explainability is documented in recent
overviews~\cite{msi-aaai22}, but also recent
publications~\cite{kutyniok-jair21,kwiatkowska-ijcai21,marquis-kr21,marquis-dke22,rubin-aaai22,amgoud-ijcai22,katz-tacas23,hcmpms-tacas23,darwiche-jlli23,lorini-jlc23}.

%% file: gxp.tex
\section{Distance-Restricted Explanations} \label{sec:drxp}

This section proposes a generalized definition of (W)AXp's and (W)CXp's,
that take into account the $\pnorm{p}$ distance between $\mbf{v}$ and
the points that can be considered in terms of changing the prediction
$c=\kappa(\mbf{v})$.
The section starts by defining AXp's/CXp's taking the $\pnorm{p}$
distance into account. Afterwards, the section proves a number of
properties motivated by the generalized definition of AXp's \& CXp's,
including MHS duality between AXp's and CXp's extends to the
distance-restricted definition of explanations.

\paragraph{Definitions.}
%
The standard definitions of AXp's \& CXp's can be extended to take a
measure $\pnorm{p}$ of distance into account.

\begin{definition}[Distance-restricted (W)AXp, $\refd$-(W)AXp]
  For a norm $\pnorm{p}$, a set of features $\fml{X}\subseteq\fml{F}$
  is a weak abductive explanation (WAXp) for an instance
  $(\mbf{v},c)$, within distance $\refd>0$ of $\mbf{v}$, if the
  following predicate holds true,
  \begin{align} \label{eq:waxpg}
    &\waxp(\fml{X};\fml{E},\refd,p) ~:=~ \forall(\mbf{x}\in\mbb{F}).\\
    & \left(\bigwedge\nolimits_{i\in\fml{X}}(x_i=v_i)\land(||\mbf{x}-\mbf{v}||_{p}\le\refd)\right)\limply(\kappa(\mbf{x})=c)\nonumber
  \end{align}
  If a (distance-restricted) weak AXp $\fml{X}$ is irreducible
  (i.e.\ it is subset-minimal), then $\fml{X}$ is a
  (distance-restricted) AXp.
\end{definition}

\begin{definition}[Distance-restricted (W)CXp, $\refd$-(W)CXp]
  For a norm $\pnorm{p}$, a set of features $\fml{Y}\subseteq\fml{F}$
  is a weak abductive explanation (WCXp) for an instance
  $(\mbf{v},c)$, within distance $\refd>0$ of $\mbf{v}$, if the
  following predicate holds true,
  \begin{align} \label{eq:wcxpg}
    &\wcxp(\fml{Y};\fml{E},\refd,p) ~:=~ \exists(\mbf{x}\in\mbb{F}). \\
    &\left(\bigwedge\nolimits_{i\in\fml{F}\setminus\fml{Y}}(x_i=v_i)\land(||\mbf{x}-\mbf{v}||_{p}\le\refd)\right)\land(\kappa(\mbf{x})\not=c) \nonumber
  \end{align}
  If a (distance-restricted) weak CXp $\fml{Y}$ is irreducible,
  then $\fml{Y}$ is a (distance-restricted) CXp.
\end{definition}

\begin{example}
  For the classifier of~\cref{ex:runex}, let the norm used be $l_1$,
  with distance value $\refd=1$. From~\cref{ex:runex:ae}, we know that
  there exist adversarial examples, e.g.\ by setting $x_1=0$ or
  $x_1=2$. However, if we fix the value of $x_1$ to 1, then any
  assignment to $x_2$ and $x_3$ with $|x_2-1|+|x_3-1|\le1$, will not
  change the prediction. As a result, $\fml{X}=\{1\}$ is a
  distance-restricted AXp when $\refd=1$. Moreover, by allowing only
  feature 1 to change value, we are able to change prediction, since
  we know there exists an adversarial example.
\end{example}

Similar to the distance-unrestricted case, the definitions above serve
to define the additional predicates $\axp$ and $\cxp$. When it is
important to distinguish distance-restricted (W)AXp's/(W)CXp's, we
will also use the notation $\refd$-(W)AXp's/$\refd$-(W)CXp's.

The AXp's (resp.~CXp's) studied in earlier
work~\cite{darwiche-ijcai18,inms-aaai19} represent a specific case of
the distance-restricted AXp's (resp.~CXp's) introduced in this section.

\begin{remark}
  Distance unrestricted AXp's (resp.~CXp's) correspond to
  $m$-distance AXp's (resp.~CXp's) for norm $l_0$.
\end{remark}

\paragraph{Properties.}
%
%
Distance-restricted explanations exhibit a number of relevant
properties.

The following observation will prove useful in designing efficient
algorithms (on the complexity of the oracle for adversarial examples)
for finding distance-restricted AXp's/CXp's.
\begin{proposition} \label{prop:monoent}
  \eqref{eq:waxpg} and \eqref{eq:wcxpg} are monotonic and up-closed.
\end{proposition}
\cref{prop:monoent} mimics a similar observation for~\cref{eq:waxp} 
and~\cref{eq:wcxp} (e.g.\ see~\cite{inams-aiia20}), and follows from
monotonicity of entailment.

Moreover, it is apparent that $\refd$-AXp's and $\refd$-CXp's offer
a rigorous definition of the concept of \emph{local} explanations
studied in non-formal XAI~\cite{molnar-bk20}.

\jnoteF{There exists a (non-empty) CXp iff there exists an AEx. There
  exists a (non-empty) AXp iff there exists an AEx.}

\begin{proposition} \label{prop:aex2xp}
  Consider an explanation problem $\fml{E}=(\fml{M},(\mbf{v},c))$ and
  some $\refd>0$ for norm $\pnorm{p}$. Let $\mbf{x}\in\mbb{F}$, with
  $||\mbf{x}-\mbf{v}||_p$, and let
  $\fml{D}=\{i\in\fml{F}\,|\,x_i\not=v_i\}$. 
  Then,
  \begin{enumerate}[nosep]
  \item If $\aex(\mbf{x};\fml{E},\refd,p)$ holds, then
    $\wcxp(\fml{D};\fml{E},\refd,p)$
    holds;
  \item If $\wcxp(\fml{D};\fml{E},\refd,p)$
    holds, then
    $\exists(\mbf{y}\in\mbb{F}).
    ||\mbf{y}-\mbf{v}||_p\le||\mbf{x}-\mbf{v}||_p\land
    \aex(\mbf{y};\fml{E},\refd,p)$.
  \end{enumerate}
\end{proposition}
  
\begin{proof}[Proof sketch] 
  We prove each claim separately.
  \begin{enumerate}[nosep]
  \item This case is immediate. Given $\refd$ and $\fml{E}$, the
    existence of an adversarial example guarantees, by definition of
    distance-restricted CXp, the existence of a distance-restricted
    CXp.
  \item By definition of distance-restricted CXp, if the features in
    $\fml{D}$ are allowed to change, given the distance restriction,
    then there exists at least one point $\mbf{y}$ such that the
    prediction changes for $\mbf{y}$, and the distance to $\mbf{y}$
    to $\mbf{v}$ does not exceed that of $\mbf{x}$.\qedhere
  \end{enumerate}
\end{proof}

Given the definitions above, we generalize~\eqref{eq:allaxp}
and~\eqref{eq:allcxp} to distance-restricted explanations, as follows:
\begin{align}
  \mbb{A}(\fml{E},\refd;p) = \{\fml{X}\subseteq\fml{F}\,|\,\axp(\fml{X};\fml{E},\refd,p)\}
  \label{eq:allaxpg}\\
  \mbb{C}(\fml{E},\refd;p) = \{\fml{X}\subseteq\fml{F}\,|\,\cxp(\fml{X};\fml{E},\refd,p)\}
  \label{eq:allcxpg} 
\end{align}


In turn, this yields the following result regarding minimal hitting
set duality between (W)AXp's and (W)CXp's.

\begin{proposition} \label{prop:duality2}
  Given an explanation problem $\fml{E}$, norm $\pnorm{p}$, and a
  value of distance $\refd>0$ then,
  \begin{enumerate}[nosep]
  \item A set $\fml{X}\subseteq\fml{F}$ is a (distance-restricted) AXp
    iff $\fml{X}$ a minimal hitting set of the CXp's in
    $\mbb{C}(\fml{E},\refd;p)$.
  \item A set $\fml{X}\subseteq\fml{F}$ is a (distance-restricted) CXp
    iff $\fml{X}$ a minimal hitting set of the AXp's in
    $\mbb{A}(\fml{E},\refd;p)$.
  \end{enumerate}
\end{proposition}
    
\begin{proof}[Proof sketch]
  We prove each claim separately.
  \begin{enumerate}[nosep]
  \item A set $\fml{X}\subseteq\fml{F}$ is a (distance-restricted) AXp
    iff $\fml{X}$ a minimal hitting set of the CXp's in
    $\mbb{C}(\fml{E},\refd;p)$.
    \begin{itemize}[nosep]
    \item[$\Rightarrow$]
      If $\fml{X}$ is an AXp, then it must be an MHS of $\mbb{C}$.\\
      Claim 1: $\fml{X}$ must hit all sets in $\mbb{C}$. If not, then the
      set of free features would allow changing the prediction; this is
      impossible since $\fml{X}$ is an AXp. \\
      Claim 2: $\fml{X}$ must be minimal; otherwise there would be a
      proper subset of $\fml{X}$ that would be a HS of $\mbb{C}$ and so a
      WAXp (see above); a contradiction.
    \item[$\Leftarrow$]
      If $\fml{H}$ is MHS of $\mbb{C}$, then it must be an AXp.\\
      We proceed as follows. From the definition of WAXP, the features in
      $\fml{H}$ are fixed to the values dictated by $\mbf{v}$. Thus, since
      we fix at least one feature in each set of $\mbb{C}$, we are
      guaranteed not to leave free a set of features representing any
      CXp. Hence, $\fml{H}$ must be a WAXP. Since $\fml{H}$ is a minimal
      hitting set, and so irreducible, then it is an AXp.
    \end{itemize}
  \item A set $\fml{X}\subseteq\fml{F}$ is a (distance-restricted) CXp
    iff $\fml{X}$ a minimal hitting set of the AXp's in
    $\mbb{A}(\fml{E},\refd;p)$.\\
    The proof mimics the previous case.\qedhere
  \end{enumerate}
\end{proof}

\cref{prop:duality2} is instrumental for the enumeration of AXp's and
CXp's, as shown in earlier work in the case of distance-unrestricted
AXp's/CXp's~\cite{inams-aiia20}, since it enables adapting well-known
algorithms for the enumeration of subset-minimal reasons of
inconsistency~\cite{lpmms-cj16}.

\begin{example}
  For the running example, we have that
  $\mbb{A}(\fml{E},1;1)=\mbb{C}(\fml{E},1;1)=\{\{1\}\}$.
\end{example}

The definitions of distance-restricted AXp's and CXp's also reveal
novel uses for abductive \& contrastive explanations.

For a given distance $\refd$, an AXp represents an irreducible
sufficient reason for the ML models not to have an adversarial
example.

The number of distance-restricted AXp's is non-decreasing with the
distance $\refd$.

\begin{proposition} \label{prop:cxpdist}
Let $0<\refd<\refdg$, then
\[\mbb{WC}(\fml{E},\refd;p)\subseteq\mbb{WC}(\fml{E},\refdg;p)\]
where $\mbb{WC}$ denotes the set of weak CXp.
We cannot deduce $\mbb{C}(\fml{E},\refd;p)\subseteq\mbb{C}(\fml{E},\refdg;p)$.
There is an counterexample by considering $l_\infty$.
Suppose we have a function $\kappa(x_1,x_2)$, where $x_1\in[0,1]$ and $x_2\in[0,1]$ and 
such that $\kappa(1,1) = 1$, $\kappa(0.5,0.5) = 0$, $\kappa(0,1) = 0$, $\kappa(1,0) = 0$.
Let us use $l_\infty$, suppose $\epsilon_1 = 0.5$ and $\epsilon_2 = 1$.
Clearly, $\epsilon_1 < \epsilon_2$. Consider the instance ((1,1),1).
If we use $\epsilon_1$, we can find an AEx $(0.5,0.5)$, from which we deduce that $\{3,4\}$ is a $0.5$-CXp.
If we use $\epsilon_2$, we can find three AEx $(0.5,0.5)$, $(1,0)$, $(0,1)$, but then we can deduce two $1$-CXps, i.e. $\{3\}$ and $\{4\}$.
\end{proposition}

%
%

%% file: rob2xp.tex
\section{From Robustness to Explainability} \label{sec:rob2xp}

This section shows that any tool for finding adversarial examples can
be used for computing distance-restricted abductive explanations,
provided the tool respects a few properties.
We detail two algorithms, both of which build on the monotonicity of 
entailment. The first one is a fairly standard solution used in
explainability, with the main difference being the use of a tool
finding adversarial examples.
The second algorithm has not been used in formal explainability, and
aims to target explanation problems with a large number of features.

\jnoteF{Extension(s):\\
  Explanation for why classifier is robust to a specific type of
  attack.}

\paragraph{Required properties of robustness tool.}
%
This section shows that distance-restricted explanations (AXp's and
CXp's) can be computed using \emph{any} existing tool for deciding the
existence of adversarial examples, provided such tool respects the
following properties:
\begin{enumerate}[nosep]
\item Features can be fixed, i.e.\ the value of some feature $i$ can
  be set to the value dictated by $\mbf{v}$, using a constraint on
  the inputs of the form $x_i=v_i$. \label{it:robt:01}
\item The robustness tool is sound, i.e.\ any reported adversarial
  example $\mbf{x}$ respects~\eqref{eq:ae}. \label{it:robt:02}
\item The robustness tool is complete, i.e.\ if there exists an
  adversarial example, i.e.\ some $\mbf{x}$ such that $\aex(\mbf{x})$
  holds true, then it will be reported. \label{it:robt:03}
\end{enumerate}

We will investigate later the impact of relaxing
property~\cref{it:robt:03}.

In the rest of the paper, we will use
$\robt(\refd,\fml{S};\fml{E},p)$ to represent a call to a
robustness oracle that decides the existence of a distance-$\refd$
adversarial example when the features in set $\fml{S}$ are fixed.

\paragraph{Computing distance-restricted AXp's.}
%
We now detail two algorithms for computing an $\refd$-AXp. These
algorithms hinge on the property that the definitions of $\refd$-AXp 
(and $\refd$-CXp) are monotonic.

\begin{algorithm}[t]
  \input{./algs/findaxp_del}
  \caption{Linear search algorithm to find AXp using adversarial
    example search}
  \label{alg:del}
\end{algorithm}

\cref{alg:del} describes a fairly standard approach for computing an
AXp, adapted to use an oracle for deciding the existence of
adversarial examples given norm $\pnorm{m}$ and value $\refd$ of
distance%
\footnote{
This algorithm mimics the well-known deletion algorithm for finding
subset-minimal explanations for infeasibility~\cite{chinneck-jc91},
but can be traced to earlier work in other
domains,e.g.~\cite{valiant-cacm84}.}.
At each step, one feature is decided not to be fixed. If the oracle
answers that an adversarial example exists, then the feature is
important and is re-added to the set $\fml{S}$. As can be concluded,
the algorithm queries the oracle $m$ times.

\begin{proposition} \label{prop:axp:del}
  \cref{alg:del} computes an AXp given the instance $(\mbf{v},c)$.
\end{proposition}

\begin{proof}[Proof sketch]
  The algorithm maintains an invariant that $\fml{S}$ is a
  distance-restricted WAXp.
  It is plain that the invariant holds after setting $\fml{S}$ to
  $\fml{F}$.
  The algorithm analyzes each feature in some order. If a feature $i$
  is allowed to change and that causes some adversarial example to be
  identified, then the feature must be fixed, since otherwise the set
  of fixed features would not represent a distance-restricted WAXp.
  The obtained set $\fml{S}$ is subset-minimal. Otherwise, if some
  feature $j$ could be dropped form $\fml{S}$, that would have been
  concluded during the execution of the For loop.
\end{proof}

Since complex ML models can be defined on a large number of features,
\cref{alg:qxp} proposes an alternative solution, which mimics the
well-known QuickXplain algorithm for finding minimal explanations for
over-constrained problems~\cite{junker-aaai04}, and which enables
reducing the number of oracle calls when the computed set is much
smaller than the starting set.
Given the tight relationship between finding AXp's and extracting
minimal unsatisfiable subsets (MUSes), then any MUS algorithm can also
be adapted to finding distance-restricted AXp's. Among others, these
include the algorithms insertion-based~\cite{puget-ecai88}, dichotomic 
search~\cite{sais-ecai06}, but also the more recent progression
algorithm~\cite{msjb-cav13}.


\begin{algorithm}[t]
  \input{./algs/findaxp_qxp}
  \caption{QuickXplain-based approach for finding AXp based on
    adversarial example search}
  \label{alg:qxp}
\end{algorithm}

With respect to the QuickXplain algorithm, the same analysis as in the
case of~\cref{prop:axp:del} gives the follow result.

\begin{proposition} \label{prop:axp:qxp}
  \cref{alg:qxp} computes an AXp given the instance $(\mbf{v},c)$. 
\end{proposition}

\paragraph{Relaxing property~\cref{it:robt:03}.}
%
The algorithms described above call an oracle for adversarial examples
a number of times. What can be done when the oracle is unable to find
an adversarial example, or prove its non-existence within a given
timeout? As the experiments confirm, this can happen when the oracle
analyzes complex problems.
Property~\cref{it:robt:03} can be relaxed as follows. (We
consider~\cref{alg:del}, but~\cref{alg:qxp} could be adapted in a
similar way.)
If the oracle for adversarial examples does not terminate within the
allocated running time for a given feature $i$, the feature is
declared to yield an adversarial example, and so it is re-added to the
set $\fml{S}$. Thus, we are guaranteed not to drop features for which
there might exist adversarial examples. Therefore, the final reported
set of features, albeit not guaranteed to be an AXp, is certainly a
WAXp.

\begin{proposition} \label{prop:axp:del:p3}
  If property~\cref{it:robt:03} does not hold, \cref{alg:del}
  computes a WAXp given the instance $(\mbf{v},c)$.
\end{proposition}

For~\cref{prop:axp:del:p3}, the quality of the computed WAXp depends
on the number of times the robustness oracle times out.

\paragraph{Contrastive explanations \& enumeration.}
%
\cref{alg:cxp:del} summarizes one algorithm for finding one
$\refd$-CXp. (The algorithm mimics ~\cref{alg:del} for finding one
$\refd$-AXp.) 
\begin{algorithm}[t]
  \input{./algs/findcxp_del}
  \caption{Linear search algorithm to find CXp using adversarial
    example search}
  \label{alg:cxp:del}
\end{algorithm}
The differences observed when compared with~\cref{alg:del} result from
the definition of $\refd$-CXp.


As the previous algorithms demonstrate, one can readily adapt the 
explainability algorithms studied in recent years. As a result,
the enumeration of AXp's and CXp's will simply exploit recent
(standard) solutions~\cite{msi-aaai22}, where the AXp/CXp extractors
will find instead $\refd$-AXp's and $\refd$-CXp's.

The same arguments used earlier can be used to conclude that,

\begin{proposition}
  \cref{alg:cxp:del} computes a CXp given the instance $(\mbf{v},c)$. 
\end{proposition}

Finally, duality enables the use of standard SAT-based algorithms for
the iterative computation of AXp's and CXp's, as illustrated in
earlier work~\cite{inams-aiia20}. It should be underlined that the
basic enumeration algorithm mimics the MARCO MUS enumeration
algorithm~\cite{lpmms-cj16}, the main difference being the algorithms
used for finding one AXp and one CXp, which are proposed in this
paper.

%% file: algs/findaxp_del.tex
\begin{flushleft}
  \hspace*{\algorithmicindent}
  \textbf{Input}: {
    Arguments: 
    $\epsilon$;
    Parameters: 
    $\fml{E}$,
    $p$}\\
  \hspace*{\algorithmicindent}
  \textbf{Output}: {One AXp $\fml{S}$}
\end{flushleft}

\begin{algorithmic}[1]
  \Function{$\findaxpdel$}{$\epsilon;\fml{E},p$}
  \State{$\fml{S}\gets\fml{F}$}
  \Comment{Initially, no feature is allowed to change}
  \For{$i\in\fml{F}$}
  \State{$\fml{S}\gets\fml{S}\setminus\{i\}$}
  \State{$\outc=\robt(\epsilon,\fml{S};\fml{E},p)$}
  \If{$\outc$}
  \State{$\fml{S}\gets\fml{S}\cup\{i\}$}
  \EndIf
  \EndFor
  \State{\tbf{return} $\fml{S}$}
  \Comment{$\fml{S}$ suffices for prediction}
  \EndFunction
\end{algorithmic}

%% file: algs/findaxp_qxp.tex
\begin{flushleft}
  \hspace*{\algorithmicindent}
  \textbf{Input}: {
    Arguments: 
    $\epsilon$;
    Parameters: 
    $\fml{E}$,
    $p$}\\
  \hspace*{\algorithmicindent}
  \textbf{Output}: {One AXp $\fml{S}$}
\end{flushleft}

\begin{algorithmic}[1]
  \Function{$\findaxpqxp$}{$\epsilon;\fml{E},p$}
  \If{\tbf{not} $\robt(\epsilon,\emptyset;\fml{E},p)$}
  \State{\tbf{return} $\emptyset$}
  \EndIf
  \State{$\fml{S}\gets\qxprecur(\emptyset,\fml{F},\False,\epsilon;\fml{E},p)$}
  \State{\tbf{return} $\fml{S}$}
  \EndFunction
  \Function{$\qxprecur$}{$\fml{B},\fml{T},\msf{newB},\epsilon;\fml{E},p$}
  \If{$\msf{newB}$ \tbf{and} \tbf{not} $\robt(\epsilon,\fml{B};\fml{E},p)$}
  \State{\tbf{return} $\emptyset$}
  \EndIf
  \If{$|\fml{T}|=1$}
  \State{\tbf{return} $\fml{T}$}
  \EndIf
  \State{$\mu\gets\lfloor\sfrac{|\fml{T}|}{2}\rfloor$}
  \State{$(\fml{T}_1,\fml{T}_2)\gets(\fml{T}_{1..\mu},\fml{T}_{\mu+1..|\fml{T}|})$}
  \State{$\fml{R}_2\gets\qxprecur(\fml{B}\cup\fml{T}_1,\fml{T}_2,|\fml{T}_1|>0,\epsilon;\fml{E},p)$}
  \State{$\fml{R}_1\gets\qxprecur(\fml{B}\cup\fml{R}_2,\fml{T}_1,|\fml{R}_2|>0,\epsilon;\fml{E},p)$}
  \State{\tbf{return} $\fml{R}_1\cup\fml{R}_2$}
  \EndFunction
\end{algorithmic}

%% file: algs/findcxp_del.tex
\begin{flushleft}
  \hspace*{\algorithmicindent}
  \textbf{Input}: {
    Arguments: 
    $\epsilon$;
    Parameters: 
    $\fml{E}$,
    $p$}\\
  \hspace*{\algorithmicindent}
  \textbf{Output}: {One CXp $\fml{S}$}
\end{flushleft}

\begin{algorithmic}[1]
  \Function{$\findcxpdel$}{$\epsilon;\fml{E},p$}
  \State{$\fml{S}\gets\fml{F}$}
  \Comment{Initially, no feature is fixed}
  \For{$i\in\fml{F}$}
  \State{$\fml{S}\gets\fml{S}\setminus\{i\}$}
  \State{$\outc=\robt(\epsilon,\fml{F}\setminus\fml{S};\fml{E},p)$}
  \If{\tbf{not}~$\outc$}
  \State{$\fml{S}\gets\fml{S}\cup\{i\}$}
  \EndIf
  \EndFor
  \State{\tbf{return} $\fml{S}$}
  \Comment{$\fml{S}$ suffices for changing prediction}
  \EndFunction
\end{algorithmic}

%% file: res.tex
\section{Experiments} \label{sec:res}

This section assesses the improvement in the scalability of computing
abductive explanations. ()
Instead of computing plain AXp's, for which existing solutions scale
up to a few tens (i.e.\ $\sim20$) of ReLU units~\cite{inms-aaai19}, we
compute $\refd$-AXp's using the Marabou~\cite{barrett-cav19}
robustness tool.
We use the same DNNs studied in earlier
work~\cite{julian2016policy,barrett-cav17,barrett-cav19},
because of their size (more than an order of magnitude larger than
what existing formal XAI approaches can explain), but also to show how
the achieved performance correlates with the practical efficiency of
the robustness tool%
\footnote{Other robustness tools and other DNNs could be considered.
Similarly, the computation of CXp's and the enumeration of AXp's/CXp's
could also be assessed. However, the goal of this section is solely to
illustrate the obtained scalability gains.}.

\input{texfigs/marabou}

\paragraph{Experimental setup.}
The evaluation includes 45 publicly available pre-trained ACAS Xu DNNs~\cite{julian2016policy}%
\footnote{https://github.com/NeuralNetworkVerification/Marabou}.
Each of these networks consists of five features and five classes. 
The DNNs are fully connected, utilizing ReLU activation functions, and comprising six hidden layers with a total of 300 ReLU nodes each.
For each DNN, we randomly select three points from the normalized feature space 
and compute one $\refd$-AXp ($\refd \in \{0.05, 0.1\}$ for each point
using~\cref{alg:del}%
\footnote{In practice, the choice of $\refd$ depends on the coverage
intended for the computed explanations or alternatively, for the
degree of robustness that is intended for the given instance.}.
(The choice of algorithm was based on the number of features of the
DNNs.)
The timeout for deciding one WAXp predicate was set to 4800 seconds.
Additionally, a prototype of the proposed algorithm was implemented in Perl. 
The robustness tool we utilized for the experiments was
Marabou~\cite{barrett-cav19}, with its default norm.
Marabou is an SMT-based tool that can answer queries about a network’s properties 
by transforming these queries into constraint satisfaction problems.
The experiments were conducted on a Linux-based machine equipped with 
a 2-Core Intel(R) Xeon(R) CPU E5-2673 v4 2.30GHz processor and 8~GB of RAM.
The machine was running Ubuntu 20.04.

\paragraph{Results.}
For the experiments~\cref{alg:del} was used, ensuring that each
feature is tested exactly once. Thus, the use of \cref{alg:del}
requires exactly five calls to compute one $\refd$-AXp.
(According to our initial tests, and because of the number of features
in the DNNs considered, the use of QuickXplain would require more
tests.)
\cref{fig:marabou}
summarizes the experimental results for the ACASXU DNNs.
(The supplementary material include more detailed results.)
%
The plots in
\cref{fig:marabou1,fig:marabou2,fig:marabou3,fig:marabou4,fig:marabou5} 
summarize the variation in running time (smallest to largest) for the
different DNNs and for the instances picked.
%
More specifically, 
each plot shows the raw performance of computing $\refd$-AXp for 9
pre-existing instances for each of the ACASXU DNNs.
%
The red curve shows the runtime (in seconds) of computing $0.05$-AXp
while the blue curve shows the runtime (in seconds) of computing $0.1$-AXp.
(Note that the runtime axis is scaled logarithmically.)
Unsurprisingly, computing $0.1$-AXp is more time consuming
as the feature space for searching counterexamples is larger.

%% file: texfigs/marabou.tex
\begin{figure*}
\centering
\begin{minipage}{0.475\textwidth}
\includegraphics[width=\linewidth]{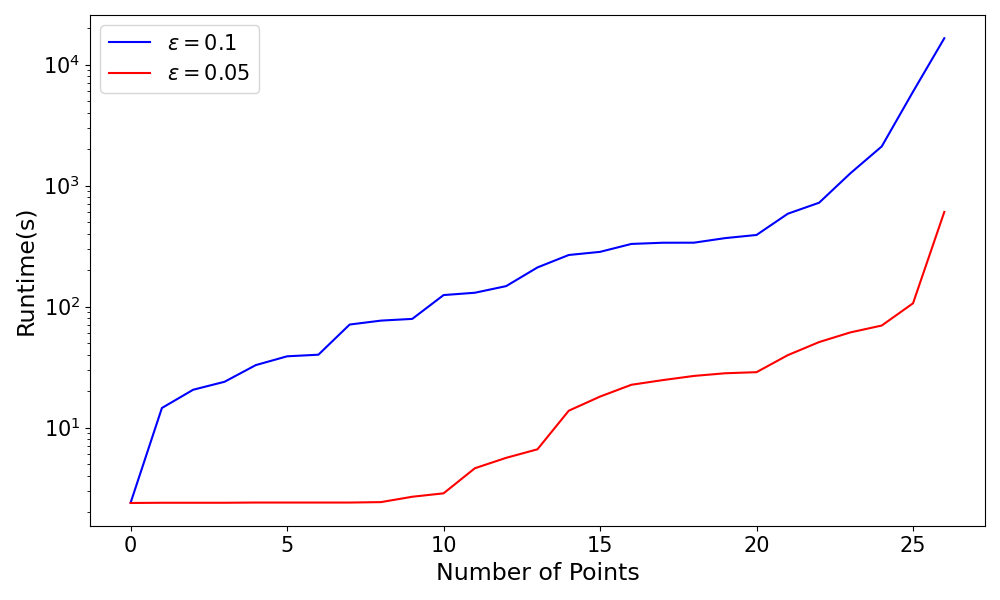} \subcaption{Runtime for ACASXU\_1} \label{fig:marabou1}
\end{minipage}\hfill
\begin{minipage}{0.475\textwidth}
\includegraphics[width=\linewidth]{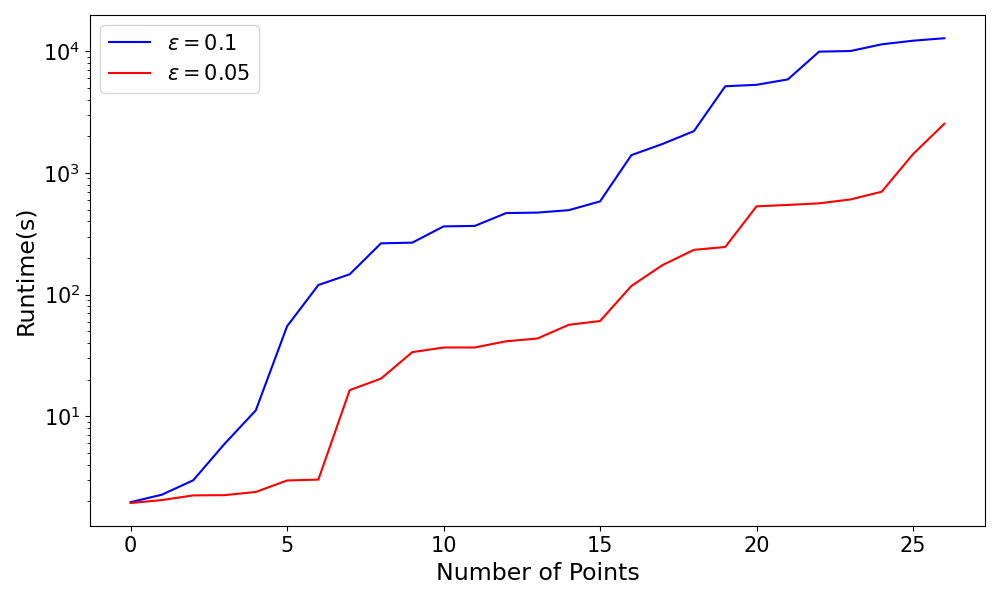} \subcaption{Runtime for ACASXU\_2} \label{fig:marabou2}
\end{minipage}

\vspace{1pt} 

\begin{minipage}{0.475\textwidth}
\includegraphics[width=\linewidth]{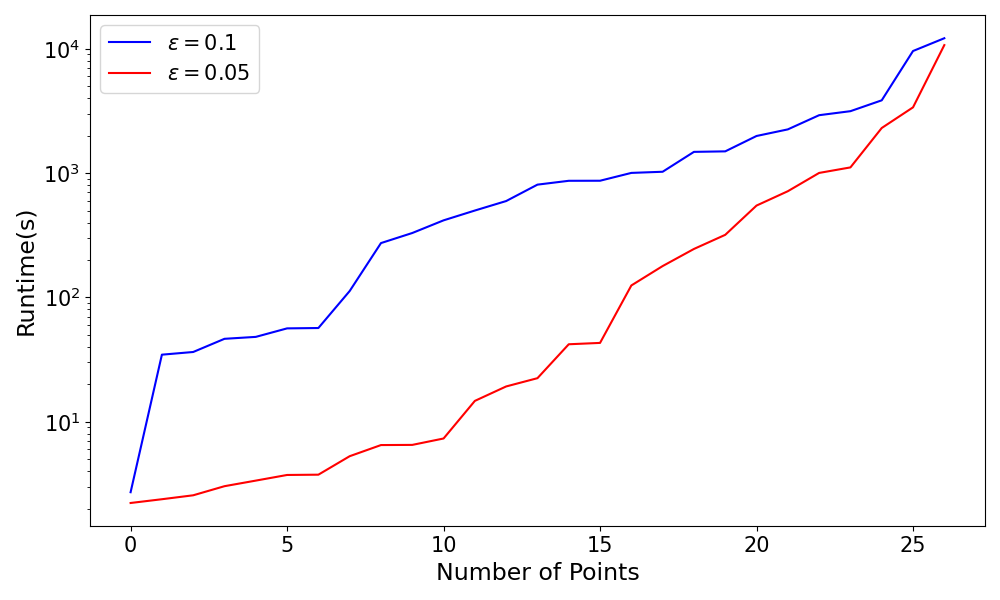} \subcaption{Runtime for ACASXU\_3} \label{fig:marabou3}
\end{minipage}\hfill
\begin{minipage}{0.475\textwidth}
\includegraphics[width=\linewidth]{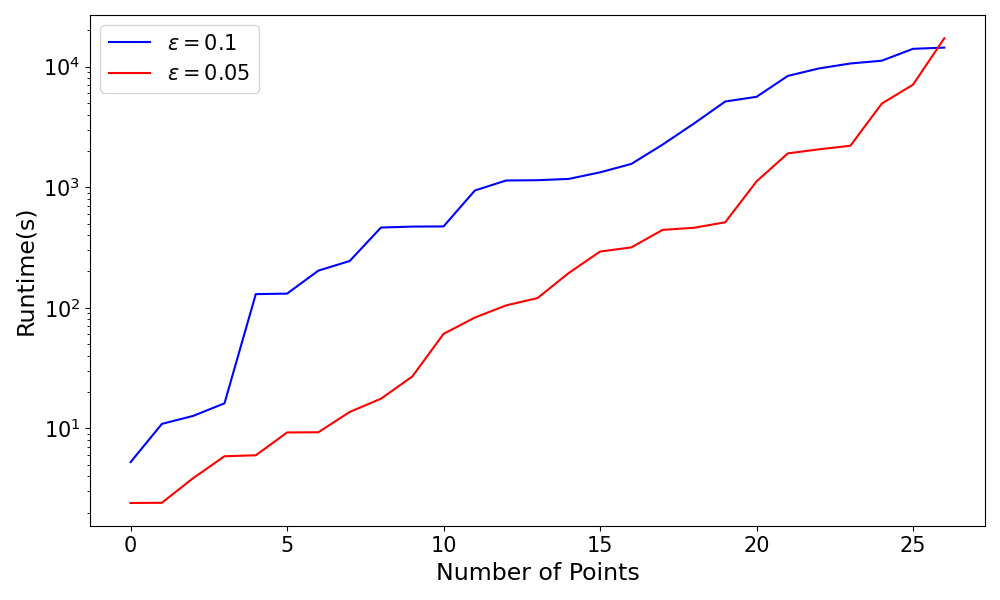} \subcaption{Runtime for ACASXU\_4} \label{fig:marabou4}
\end{minipage}

\vspace{1pt} 

\begin{minipage}{0.475\textwidth}
\includegraphics[width=\linewidth]{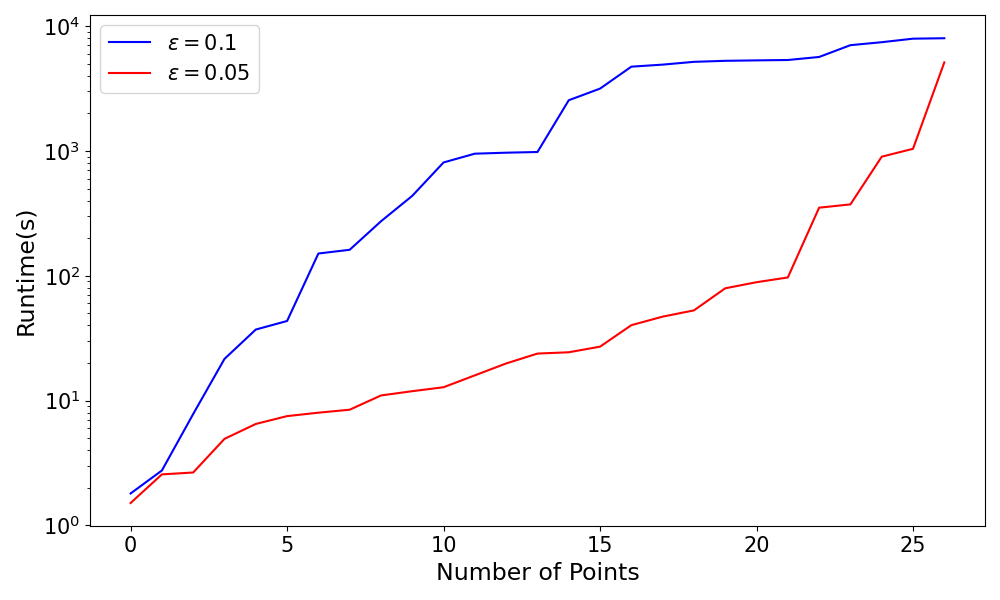} \subcaption{Runtime for ACASXU\_5} \label{fig:marabou5}
\end{minipage}

\caption{Runtime of Computing $\epsilon$-AXp for ACASXU DNNs} \label{fig:marabou}
\end{figure*}

%% file: conc.tex
\section{Conclusions} \label{sec:conc}

This paper demonstrates the existence of a tight relationship between
adversarial examples and (rigorous) abductive (and contrastive)
explanations.
This relationship enables the use of off-the-shelf robustness
algorithms for finding formal explanations, namely abductive
explanations.
Furthermore, the paper proposes algorithms for computing abductive (or
contrastive) explanations that use an oracle for adversarial examples,
and illustrates how explanations (namely abductive explanations) can
be enumerated using existing algorithms.
These results hinge on generalized definitions of formal explanations
to account for bounds on norm $\pnorm{p}$ distance when computing
explanations.

The experiments illustrate the gains in scalability that are obtained
by using existing tools for finding adversarial examples.
Whereas the current state of the art for rigorously explaining neural
networks is bounded by a few tens of ReLU
units~\cite{inms-aaai19,katz-tacas23}, this paper shows that
distance-restricted abductive explanations can be computed for neural
networks with several hundreds of ReLU units, i.e.\ more than an order
of magnitude improvement.
%
%
The obtained gains in scalability address in part one of the main
drawbacks of formal XAI, thereby allowing the use of formal
explanations with significantly more complex ML models, and opening
the perspective of further scalability gains in the near future.
More importantly, the results in this paper bind future improvements
in reasoning about formal explanations to future improvements in
assessing robustness of ML models. In other words, any complex ML
model, for which robustness can be assessed with a tool respecting
the properties outlined in this paper, can in principle be explained.
(taking into account the worst-case number of features that may have
to be analyzed).
Furthermore, it should be highlighted that these observations apply to
\emph{any} ML model, besides those already studied in earlier work on
formal explainability. Besides DNNs with ReLU units, other examples
of ML models include other types of DNNs, SVMs, among others.

A number of lines of research can be envisioned, most related with
tapping on the vast number of existing robustness tools, and assessing
their scalability in the context of explainability. For example, since
some robustness tools offer probabilistic guarantees, it will be
important to assess whether such tools can also be exploited in
computing explanations.
A number of explainability queries have been studied in recent
years~\cite{marquis-kr21}. More efficient reasoning will enable
answering such queries in practice.

%

%% file: acks.tex
\paragraph{Acknowledgments.}
This work was supported by the AI Interdisciplinary Institute ANITI,
funded by the French program ``Investing for the Future -- PIA3''
under Grant agreement no.\ ANR-19-PI3A-0004,
and 
by the H2020-ICT38 project COALA ``Cognitive Assisted agile
manufacturing for a Labor force supported by trustworthy Artificial
intelligence''.
%
The work was also influenced by discussions with several colleagues,
including the researchers and collaborators of ANITI's 
DeepLever Chair.
N.\ Narodytska spotted a typo in the pseudo-code of the deletion-based
algorithms. 
\input{acks_JMS}

%% file: acks_JMS.tex
JMS also acknowledges the incentive provided by the ERC who, by not
funding this research nor a handful of other grant applications
between 2012 and 2022, has had a lasting impact in framing the
research presented in this paper.

%% file: replbib.tex
\newtoggle{mkbbl}

%% file: togbbl.tex
\settoggle{mkbbl}{false}

%% file: appendix.tex
\section{Appendix} \label{sec:app}

\subsection{More Results}

\input{tabs/marabou}

When the value of $\refd$ is sufficiently small, it is more likely that the DNNs are robust against perturbations, and as a result, there may not exist any AXp. 
This can be illustrated by considering the case of \emph{N3\_1} when $\refd = 0.05$ is selected.
Moreover, it should be noted that if there is no $\refd$-AXp for a certain value of $\refd$, it also implies that there is no $\refd'$-AXp 
for any $\refd' < \refd$. 
For example, in the case of \emph{N1\_1} and the first point, there is no $0.1$-AXp, so there is no $0.05$-AXp either.
The runtime of computing $\refd$-AXp is influenced by two factors: 
1) the choice of $\refd$, and 
2) the number of dropped features during the computation of $\refd$-AXp.
As the value of $\refd$ increases, or as the number of dropped features increases,
the runtime of computing one $\refd$-AXp also increases.
This is because larger values of $\refd$ or a larger number of dropped features 
imply a larger search space for finding adversarial examples, which requires more computational effort.
The choice of $\refd$ has a greater impact on the runtime.
In most cases, computing one $0.1$-AXp consumes much more time compared to computing one $0.05$-AXp.
However, the impact of the number of dropped features cannot be overlooked.
For example, in the case of \emph{N2\_4} and the third point, in the case of \emph{N3\_7} and the second points,
as well as in the case of \emph{N4\_5} and the third point,
the runtime of computing one $0.05$-AXp is much larger than that of computing one $0.1$-AXp. 
This highlights the influence of the number of dropped features on the computation time.

%% file: tabs/marabou.tex
\begin{table}
\setlength{\tabcolsep}{2pt}
\centering
\caption{
Statistics for computing $\refd$-AXp with $\refd = \{0.1, 0.05\}$ for ACASXU\_1 DNNs.
The column {$|\fml{X}|$} indicates the size of AXp. 
The column {Time} reports the runtime required to compute one $\refd$-AXp.
The column {\#TO} indicates the number of timeouts that occurred during the process of deciding WAXp predicate using the robustness tool.
}
\label{tab:marabou1}
\begin{tabular}{crcrr|crr}
\toprule[1.2pt]
DNN & pts & $|\fml{X}|$ & Time & \#TO & $|\fml{X}|$ & Time & \#TO \\ \midrule[0.8pt]
\multicolumn{5}{c}{$\refd=0.1$} & \multicolumn{3}{c}{$\refd=0.05$} \\
\midrule[0.5pt]
\multirow{3}{*}{N1\_1} & \#1 & 0 & 2.39 & 0 & 0 & 2.39 & 0 \\
 & \#2 & 4 & 32.82 & 0 & 0 & 2.68 & 0 \\
 & \#3 & 1 & 1264.17 & 0 & 1 & 26.71 & 0 \\
\midrule
\multirow{3}{*}{N1\_2} & \#1 & 4 & 147.56 & 0 & 3 & 28.69 & 0 \\
 & \#2 & 0 & 721.40 & 0 & 0 & 2.38 & 0 \\
 & \#3 & 0 & 585.31 & 0 & 0 & 50.88 & 0 \\
\midrule
\multirow{3}{*}{N1\_3} & \#1 & 0 & 210.54 & 0 & 0 & 18.03 & 0 \\
 & \#2 & 0 & 2104.68 & 0 & 0 & 2.86 & 0 \\
 & \#3 & 4 & 20.54 & 0 & 3 & 6.61 & 0 \\
\midrule
\multirow{3}{*}{N1\_4} & \#1 & 2 & 266.79 & 0 & 0 & 39.67 & 0 \\
 & \#2 & 0 & 23.89 & 0 & 0 & 2.40 & 0 \\
 & \#3 & 0 & 40.00 & 0 & 0 & 2.40 & 0 \\
\midrule
\multirow{3}{*}{N1\_5} & \#1 & 3 & 79.09 & 0 & 2 & 61.17 & 0 \\
 & \#2 & 2 & 130.07 & 0 & 1 & 13.78 & 0 \\
 & \#3 & 0 & 329.24 & 0 & 0 & 2.39 & 0 \\
\midrule
\multirow{3}{*}{N1\_6} & \#1 & 0 & 367.90 & 0 & 0 & 2.39 & 0 \\
 & \#2 & 5 & 14.51 & 0 & 5 & 5.62 & 0 \\
 & \#3 & 1 & 16504.72 & 1 & 0 & 4.61 & 0 \\
\midrule
\multirow{3}{*}{N1\_7} & \#1 & 1 & 337.19 & 0 & 0 & 69.68 & 0 \\
 & \#2 & 4 & 38.80 & 0 & 3 & 28.10 & 0 \\
 & \#3 & 2 & 124.43 & 0 & 2 & 106.39 & 0 \\
\midrule
\multirow{3}{*}{N1\_8} & \#1 & 3 & 5966.70 & 0 & 2 & 605.71 & 0 \\
 & \#2 & 0 & 390.52 & 0 & 0 & 2.40 & 0 \\
 & \#3 & 1 & 283.17 & 0 & 0 & 24.66 & 0 \\
\midrule
\multirow{3}{*}{N1\_9} & \#1 & 2 & 76.53 & 0 & 3 & 22.57 & 0 \\
 & \#2 & 0 & 71.03 & 0 & 0 & 2.42 & 0 \\
 & \#3 & 0 & 336.87 & 0 & 0 & 2.40 & 0 \\
\bottomrule[1.2pt]
\end{tabular}
\end{table}


\begin{table}[]
\setlength{\tabcolsep}{2pt}
\centering
\caption{
Statistics for computing $\refd$-AXp for ACASXU\_2 DNNs.
}
\label{tab:marabou2}
\begin{tabular}{crcrr|crr}
\toprule[1.2pt]
DNN & pts & $|\fml{X}|$ & Time & \#TO & $|\fml{X}|$ & Time & \#TO \\ \midrule[0.8pt]
\multicolumn{5}{c}{$\refd=0.1$} & \multicolumn{3}{c}{$\refd=0.05$} \\
\midrule[0.5pt]
\multirow{3}{*}{N2\_1} & \#1 & 4 & 2.27 & 0 & 3 & 2.05 & 0 \\
 & \#2 & 3 & 467.72 & 0 & 4 & 41.36 & 0 \\
 & \#3 & 2 & 146.88 & 0 & 2 & 60.67 & 0 \\
\midrule
\multirow{3}{*}{N2\_2} & \#1 & 4 & 1.97 & 0 & 3 & 2.97 & 0 \\
 & \#2 & 1 & 5151.53 & 1 & 0 & 175.10 & 0 \\
 & \#3 & 2 & 11.19 & 0 & 1 & 16.41 & 0 \\
\midrule
\multirow{3}{*}{N2\_3} & \#1 & 2 & 267.29 & 0 & 1 & 36.79 & 0 \\
 & \#2 & 1 & 10034.32 & 1 & 0 & 1426.61 & 0 \\
 & \#3 & 2 & 5865.55 & 0 & 2 & 246.23 & 0 \\
\midrule
\multirow{3}{*}{N2\_4} & \#1 & 3 & 54.96 & 0 & 4 & 56.48 & 0 \\
 & \#2 & 0 & 263.82 & 0 & 0 & 43.59 & 0 \\
 & \#3 & 4 & 5.93 & 0 & 2 & 36.76 & 0 \\
\midrule
\multirow{3}{*}{N2\_5} & \#1 & 0 & 367.01 & 0 & 0 & 605.17 & 0 \\
 & \#2 & 4 & 363.43 & 0 & 4 & 117.74 & 0 \\
 & \#3 & 1 & 1399.72 & 0 & 1 & 3.02 & 0 \\
\midrule
\multirow{3}{*}{N2\_6} & \#1 & 1 & 12198.47 & 1 & 0 & 233.30 & 0 \\
 & \#2 & 2 & 11391.86 & 1 & 1 & 562.43 & 0 \\
 & \#3 & 2 & 9921.77 & 2 & 0 & 700.76 & 0 \\
\midrule
\multirow{3}{*}{N2\_7} & \#1 & 4 & 2.98 & 0 & 4 & 1.94 & 0 \\
 & \#2 & 0 & 2205.67 & 0 & 0 & 2.39 & 0 \\
 & \#3 & 1 & 5296.30 & 1 & 0 & 33.67 & 0 \\
\midrule
\multirow{3}{*}{N2\_8} & \#1 & 2 & 12782.04 & 2 & 1 & 531.32 & 0 \\
 & \#2 & 2 & 583.23 & 0 & 2 & 545.73 & 0 \\
 & \#3 & 3 & 472.02 & 0 & 2 & 20.40 & 0 \\
\midrule
\multirow{3}{*}{N2\_9} & \#1 & 0 & 119.93 & 0 & 2 & 2534.52 & 0 \\
 & \#2 & 2 & 494.13 & 0 & 0 & 2.25 & 0 \\
 & \#3 & 0 & 1731.25 & 0 & 0 & 2.24 & 0 \\
\bottomrule[1.2pt]
\end{tabular}
\end{table}


\begin{table}[]
\setlength{\tabcolsep}{2pt}
\centering
\caption{
Statistics for computing $\refd$-AXp for ACASXU\_3 DNNs.
}
\label{tab:marabou3}
\begin{tabular}{crcrr|crr}
\toprule[1.2pt]
DNN & pts & $|\fml{X}|$ & Time & \#TO & $|\fml{X}|$ & Time & \#TO \\ \midrule[0.8pt]
\multicolumn{5}{c}{$\refd=0.1$} & \multicolumn{3}{c}{$\refd=0.05$} \\
\midrule[0.5pt]
\multirow{3}{*}{N3\_1} & \#1 & 0 &499.75 & 0 & 0 &2.38 & 0 \\
 & \#2 & 1 &595.71 & 0 & 0 &318.62 & 0 \\
 & \#3 & 1 &112.54 & 0 & 0 &43.05 & 0 \\
\midrule
\multirow{3}{*}{N3\_2} & \#1 & 3 &12150.36 & 2 & 2 &10701.09 & 2 \\
 & \#2 & 3 &34.62 & 0 & 2 &14.71 & 0 \\
 & \#3 & 2 &273.49 & 0 & 2 &19.22 & 0 \\
\midrule
\multirow{3}{*}{N3\_3} & \#1 & 0 &48.13 & 0 & 0 &2.22 & 0 \\
 & \#2 & 0 &866.77 & 0 & 0 &2.56 & 0 \\
 & \#3 & 3 &56.72 & 0 & 2 &3.73 & 0 \\
\midrule
\multirow{3}{*}{N3\_4} & \#1 & 2 &416.67 & 0 & 4 &3.36 & 0 \\
 & \#2 & 1 &867.77 & 0 & 0 &22.40 & 0 \\
 & \#3 & 2 &2246.29 & 0 & 0 &5.28 & 0 \\
\midrule
\multirow{3}{*}{N3\_5} & \#1 & 4 &36.38 & 0 & 2 &41.98 & 0 \\
 & \#2 & 3 &1024.77 & 0 & 1 &2304.11 & 0 \\
 & \#3 & 2 &2921.13 & 1 & 2 &245.29 & 0 \\
\midrule
\multirow{3}{*}{N3\_6} & \#1 & 0 &329.57 & 0 & 0 &6.51 & 0 \\
 & \#2 & 3 &807.24 & 0 & 2 &547.89 & 0 \\
 & \#3 & 1 &3146.86 & 0 & 1 &714.69 & 0 \\
\midrule
\multirow{3}{*}{N3\_7} & \#1 & 3 &1481.07 & 0 & 4 &3.03 & 0 \\
 & \#2 & 4 &46.44 & 0 & 1 &1110.75 & 0 \\
 & \#3 & 2 &1495.22 & 0 & 1 &3378.09 & 1 \\
\midrule
\multirow{3}{*}{N3\_8} & \#1 & 2 &9589.72 & 1 & 0 &6.49 & 0 \\
 & \#2 & 0 &56.34 & 0 & 0 &1002.70 & 0 \\
 & \#3 & 2 &3847.70 & 0 & 0 &178.68 & 0 \\
\midrule
\multirow{3}{*}{N3\_9} & \#1 & 2 &1987.55 & 0 & 0 &124.77 & 0 \\
 & \#2 & 2 &2.71 & 0 & 1 &3.75 & 0 \\
 & \#3 & 1 &1003.20 & 0 & 0 &7.33 & 0 \\
\bottomrule[1.2pt]
\end{tabular}
\end{table}


\begin{table}[]
\setlength{\tabcolsep}{2pt}
\centering
\caption{
Statistics for computing $\refd$-AXp for ACASXU\_4 DNNs.
}
\label{tab:marabou4}
\begin{tabular}{crcrr|crr}
\toprule[1.2pt]
DNN & pts & $|\fml{X}|$ & Time & \#TO & $|\fml{X}|$ & Time & \#TO \\ \midrule[0.8pt]
\multicolumn{5}{c}{$\refd=0.1$} & \multicolumn{3}{c}{$\refd=0.05$} \\
\midrule[0.5pt]
\multirow{3}{*}{N4\_1} & \#1 & 0 &1560.59 & 0 & 0 &1117.69 & 0 \\
 & \#2 & 2 &8386.59 & 0 & 0 &292.50 & 0 \\
 & \#3 & 4 &203.31 & 0 & 4 &104.64 & 0 \\
\midrule
\multirow{3}{*}{N4\_2} & \#1 & 4 &10.88 & 0 & 4 &9.23 & 0 \\
 & \#2 & 3 &244.13 & 0 & 1 &13.66 & 0 \\
 & \#3 & 0 &1137.18 & 0 & 0 &194.22 & 0 \\
\midrule
\multirow{3}{*}{N4\_3} & \#1 & 2 &5142.96 & 1 & 1 &511.82 & 0 \\
 & \#2 & 1 &462.88 & 0 & 1 &26.87 & 0 \\
 & \#3 & 2 &3374.66 & 0 & 0 &460.71 & 0 \\
\midrule
\multirow{3}{*}{N4\_4} & \#1 & 2 &1328.44 & 0 & 2 &316.78 & 0 \\
 & \#2 & 0 &1172.15 & 0 & 0 &442.84 & 0 \\
 & \#3 & 4 &129.74 & 0 & 3 &3.86 & 0 \\
\midrule
\multirow{3}{*}{N4\_5} & \#1 & 3 &14066.93 & 2 & 1 &4933.66 & 1 \\
 & \#2 & 3 &16.11 & 0 & 3 &9.27 & 0 \\
 & \#3 & 2 &5623.34 & 1 & 1 &17220.40 & 1 \\
\midrule
\multirow{3}{*}{N4\_6} & \#1 & 3 &130.97 & 0 & 2 &5.86 & 0 \\
 & \#2 & 0 &2257.99 & 0 & 0 &2.41 & 0 \\
 & \#3 & 1 &939.59 & 0 & 0 &82.94 & 0 \\
\midrule
\multirow{3}{*}{N4\_7} & \#1 & 3 &5.24 & 0 & 3 &120.24 & 0 \\
 & \#2 & 3 &12.67 & 0 & 3 &5.97 & 0 \\
 & \#3 & 1 &1143.34 & 0 & 0 &17.60 & 0 \\
\midrule
\multirow{3}{*}{N4\_8} & \#1 & 3 &14398.49 & 2 & 2 &2209.18 & 0 \\
 & \#2 & 1 &471.29 & 0 & 0 &60.75 & 0 \\
 & \#3 & 2 &10629.91 & 2 & 0 &1908.19 & 0 \\
\midrule
\multirow{3}{*}{N4\_9} & \#1 & 2 &11207.43 & 1 & 2 &2063.07 & 0 \\
 & \#2 & 0 &472.86 & 0 & 0 &2.40 & 0 \\
 & \#3 & 2 &9677.28 & 1 & 2 &7092.04 & 1 \\
\bottomrule[1.2pt]
\end{tabular}
\end{table}


\begin{table}[]
\setlength{\tabcolsep}{2pt}
\centering
\caption{
Statistics for computing $\refd$-AXp for ACASXU\_5 DNNs.
}
\label{tab:marabou5}
\begin{tabular}{crcrr|crr}
\toprule[1.2pt]
DNN & pts & $|\fml{X}|$ & Time & \#TO & $|\fml{X}|$ & Time & \#TO \\ \midrule[0.8pt]
\multicolumn{5}{c}{$\refd=0.1$} & \multicolumn{3}{c}{$\refd=0.05$} \\
\midrule[0.5pt]
\multirow{3}{*}{N5\_1} & \#1 & 3 &43.38 & 0 & 3 &47.01 & 0 \\
 & \#2 & 3 &980.40 & 0 & 2 &1040.17 & 0 \\
 & \#3 & 0 &950.57 & 0 & 0 &10.98 & 0 \\
\midrule
\multirow{3}{*}{N5\_2} & \#1 & 1 &5273.58 & 1 & 0 &23.78 & 0 \\
 & \#2 & 1 &3158.12 & 1 & 0 &6.49 & 0 \\
 & \#3 & 3 &37.05 & 0 & 2 &27.02 & 0 \\
\midrule
\multirow{3}{*}{N5\_3} & \#1 & 0 &161.38 & 0 & 0 &2.65 & 0 \\
 & \#2 & 3 &436.70 & 0 & 2 &88.74 & 0 \\
 & \#3 & 2 &2.75 & 0 & 2 &4.93 & 0 \\
\midrule
\multirow{3}{*}{N5\_4} & \#1 & 5 &1.80 & 0 & 4 &1.51 & 0 \\
 & \#2 & 1 &7935.65 & 1 & 0 &79.31 & 0 \\
 & \#3 & 2 &808.67 & 0 & 3 &52.76 & 0 \\
\midrule
\multirow{3}{*}{N5\_5} & \#1 & 0 &968.08 & 0 & 0 &8.44 & 0 \\
 & \#2 & 1 &5175.69 & 1 & 0 &15.91 & 0 \\
 & \#3 & 4 &7.80 & 0 & 4 &2.56 & 0 \\
\midrule
\multirow{3}{*}{N5\_6} & \#1 & 1 &4916.43 & 1 & 1 &5124.36 & 1 \\
 & \#2 & 1 &7434.66 & 1 & 0 &12.78 & 0 \\
 & \#3 & 2 &5353.69 & 0 & 0 &40.16 & 0 \\
\midrule
\multirow{3}{*}{N5\_7} & \#1 & 2 &150.67 & 0 & 2 &11.88 & 0 \\
 & \#2 & 2 &4734.08 & 1 & 0 &7.99 & 0 \\
 & \#3 & 4 &21.56 & 0 & 3 &7.50 & 0 \\
\midrule
\multirow{3}{*}{N5\_8} & \#1 & 3 &273.03 & 0 & 3 &24.36 & 0 \\
 & \#2 & 2 &2548.19 & 0 & 3 &350.95 & 0 \\
 & \#3 & 2 &7999.33 & 1 & 1 &899.21 & 0 \\
\midrule
\multirow{3}{*}{N5\_9} & \#1 & 1 &5317.12 & 1 & 0 &19.82 & 0 \\
 & \#2 & 2 &7039.75 & 1 & 1 &96.85 & 0 \\
 & \#3 & 1 &5661.11 & 1 & 0 &373.02 & 0 \\
\bottomrule[1.2pt]
\end{tabular}
\end{table}

%% file: paper.bbl
\begin{thebibliography}{10}

\bibitem{kohli-nips19b}
J.~Alayrac, J.~Uesato, P.~Huang, A.~Fawzi, R.~Stanforth, and P.~Kohli.
\newblock Are labels required for improving adversarial robustness?
\newblock In {\em NeurIPS}, pages 12192--12202, 2019.

\bibitem{amgoud-ijcai22}
L.~Amgoud and J.~Ben{-}Naim.
\newblock Axiomatic foundations of explainability.
\newblock In {\em IJCAI}, pages 636--642, 2022.

\bibitem{barcelo-nips22}
M.~Arenas, P.~Barcel{\'{o}}, M.~Romero, and B.~Subercaseaux.
\newblock On computing probabilistic explanations for decision trees.
\newblock In {\em NeurIPS}, 2022.

\bibitem{carlini-icml18}
A.~Athalye, N.~Carlini, and D.~A. Wagner.
\newblock Obfuscated gradients give a false sense of security: Circumventing
  defenses to adversarial examples.
\newblock In {\em ICML}, pages 274--283, 2018.

\bibitem{marquis-kr21}
G.~Audemard, S.~Bellart, L.~Bounia, F.~Koriche, J.~Lagniez, and P.~Marquis.
\newblock On the computational intelligibility of boolean classifiers.
\newblock In {\em KR}, pages 74--86, 2021.

\bibitem{marquis-dke22}
G.~Audemard, S.~Bellart, L.~Bounia, F.~Koriche, J.~Lagniez, and P.~Marquis.
\newblock On the explanatory power of boolean decision trees.
\newblock {\em Data Knowl. Eng.}, 142:102088, 2022.

\bibitem{muller-plosone15}
S.~Bach, A.~Binder, G.~Montavon, F.~Klauschen, K.-R. M{\"u}ller, and W.~Samek.
\newblock On pixel-wise explanations for non-linear classifier decisions by
  layer-wise relevance propagation.
\newblock {\em PloS one}, 10(7):e0130140, 2015.

\bibitem{vechev-nips19a}
M.~Balunovic, M.~Baader, G.~Singh, T.~Gehr, and M.~T. Vechev.
\newblock Certifying geometric robustness of neural networks.
\newblock In {\em NeurIPS}, pages 15287--15297, 2019.

\bibitem{katz-tacas23}
S.~Bassan and G.~Katz.
\newblock Towards formal {XAI:} formally approximate minimal explanations of
  neural networks.
\newblock In {\em TACAS}, pages 187--207, 2023.

\bibitem{vechev-icml20}
P.~Bielik and M.~T. Vechev.
\newblock Adversarial robustness for code.
\newblock In {\em ICML}, pages 896--907, 2020.

\bibitem{kumar-jmlr20}
R.~Bunel, J.~Lu, I.~Turkaslan, P.~H.~S. Torr, P.~Kohli, and M.~P. Kumar.
\newblock Branch and bound for piecewise linear neural network verification.
\newblock {\em J. Mach. Learn. Res.}, 21:42:1--42:39, 2020.

\bibitem{kumar-uai20}
R.~Bunel, A.~D. Palma, A.~Desmaison, K.~Dvijotham, P.~Kohli, P.~H.~S. Torr, and
  M.~P. Kumar.
\newblock Lagrangian decomposition for neural network verification.
\newblock In {\em UAI}, pages 370--379, 2020.

\bibitem{carlini-sp17}
N.~Carlini and D.~A. Wagner.
\newblock Towards evaluating the robustness of neural networks.
\newblock In {\em IEEE S\&P}, pages 39--57, 2017.

\bibitem{chinneck-jc91}
J.~W. Chinneck and E.~W. Dravnieks.
\newblock Locating minimal infeasible constraint sets in linear programs.
\newblock {\em {INFORMS} J. Comput.}, 3(2):157--168, 1991.

\bibitem{kolter-icml19}
J.~M. Cohen, E.~Rosenfeld, and J.~Z. Kolter.
\newblock Certified adversarial robustness via randomized smoothing.
\newblock In {\em ICML}, pages 1310--1320, 2019.

\bibitem{hein-icml20}
F.~Croce and M.~Hein.
\newblock Reliable evaluation of adversarial robustness with an ensemble of
  diverse parameter-free attacks.
\newblock In {\em ICML}, pages 2206--2216, 2020.

\bibitem{darwiche-jlli23}
A.~Darwiche and A.~Hirth.
\newblock On the (complete) reasons behind decisions.
\newblock {\em J. Log. Lang. Inf.}, 32(1):63--88, 2023.

\bibitem{puget-ecai88}
J.~L. de~Siqueira~N. and J.-F. Puget.
\newblock Explanation-based generalisation of failures.
\newblock In {\em ECAI}, pages 339--344, 1988.

\bibitem{vechev-iclr22a}
D.~I. Dimitrov, G.~Singh, T.~Gehr, and M.~T. Vechev.
\newblock Provably robust adversarial examples.
\newblock In {\em ICLR}, 2022.

\bibitem{ehlers-atva17}
R.~Ehlers.
\newblock Formal verification of piece-wise linear feed-forward neural
  networks.
\newblock In {\em ATVA}, pages 269--286, 2017.

\bibitem{vechev-iclr22b}
C.~Ferrari, M.~N. M{\"{u}}ller, N.~Jovanovic, and M.~T. Vechev.
\newblock Complete verification via multi-neuron relaxation guided
  branch-and-bound.
\newblock In {\em ICLR}, 2022.

\bibitem{vechev-sp18}
T.~Gehr, M.~Mirman, D.~Drachsler{-}Cohen, P.~Tsankov, S.~Chaudhuri, and M.~T.
  Vechev.
\newblock {AI2:} safety and robustness certification of neural networks with
  abstract interpretation.
\newblock In {\em SP}, pages 3--18, 2018.

\bibitem{rubin-aaai22}
N.~Gorji and S.~Rubin.
\newblock Sufficient reasons for classifier decisions in the presence of domain
  constraints.
\newblock In {\em {AAAI}}, February 2022.

\bibitem{kohli-iccv19}
S.~Gowal, K.~Dvijotham, R.~Stanforth, R.~Bunel, C.~Qin, J.~Uesato,
  R.~Arandjelovic, T.~A. Mann, and P.~Kohli.
\newblock Scalable verified training for provably robust image classification.
\newblock In {\em ICCV}, pages 4841--4850, 2019.

\bibitem{sais-ecai06}
F.~Hemery, C.~Lecoutre, L.~Sais, and F.~Boussemart.
\newblock Extracting {MUCs} from constraint networks.
\newblock In {\em ECAI}, pages 113--117, 2006.

\bibitem{horn-bk12}
R.~A. Horn and C.~R. Johnson.
\newblock {\em Matrix analysis}.
\newblock Cambridge university press, 2012.

\bibitem{kohli-emnlp19}
P.~Huang, R.~Stanforth, J.~Welbl, C.~Dyer, D.~Yogatama, S.~Gowal, K.~Dvijotham,
  and P.~Kohli.
\newblock Achieving verified robustness to symbol substitutions via interval
  bound propagation.
\newblock In {\em EMNLP-IJCNLP}, pages 4081--4091, 2019.

\bibitem{hcmpms-tacas23}
X.~Huang, M.~C. Cooper, A.~Morgado, J.~Planes, and J.~Marques{-}Silva.
\newblock Feature necessity {\&} relevancy in {ML} classifier explanations.
\newblock In {\em TACAS}, pages 167--186, 2023.

\bibitem{kwiatkowska-cav17}
X.~Huang, M.~Kwiatkowska, S.~Wang, and M.~Wu.
\newblock Safety verification of deep neural networks.
\newblock In {\em CAV}, pages 3--29, 2017.

\bibitem{ignatiev-ijcai20}
A.~Ignatiev.
\newblock Towards trustable explainable {AI}.
\newblock In {\em IJCAI}, pages 5154--5158, 2020.

\bibitem{inams-aiia20}
A.~Ignatiev, N.~Narodytska, N.~Asher, and J.~Marques{-}Silva.
\newblock From contrastive to abductive explanations and back again.
\newblock In {\em AIxIA}, pages 335--355, 2020.

\bibitem{inms-aaai19}
A.~Ignatiev, N.~Narodytska, and J.~Marques{-}Silva.
\newblock Abduction-based explanations for machine learning models.
\newblock In {\em AAAI}, pages 1511--1519, 2019.

\bibitem{inms-nips19}
A.~Ignatiev, N.~Narodytska, and J.~Marques{-}Silva.
\newblock On relating explanations and adversarial examples.
\newblock In {\em NeurIPS}, pages 15857--15867, 2019.

\bibitem{julian2016policy}
K.~D. Julian, J.~Lopez, J.~S. Brush, M.~P. Owen, and M.~J. Kochenderfer.
\newblock Policy compression for aircraft collision avoidance systems.
\newblock In {\em DASC}, pages 1--10, 2016.

\bibitem{junker-aaai04}
U.~Junker.
\newblock {QUICKXPLAIN:} preferred explanations and relaxations for
  over-constrained problems.
\newblock In {\em AAAI}, pages 167--172, 2004.

\bibitem{barrett-cav17}
G.~Katz, C.~W. Barrett, D.~L. Dill, K.~Julian, and M.~J. Kochenderfer.
\newblock Reluplex: An efficient {SMT} solver for verifying deep neural
  networks.
\newblock In {\em CAV}, pages 97--117, 2017.

\bibitem{barrett-cav19}
G.~Katz, D.~A. Huang, D.~Ibeling, K.~Julian, C.~Lazarus, R.~Lim, P.~Shah,
  S.~Thakoor, H.~Wu, A.~Zeljic, D.~L. Dill, M.~J. Kochenderfer, and C.~W.
  Barrett.
\newblock The marabou framework for verification and analysis of deep neural
  networks.
\newblock In {\em CAV}, pages 443--452, 2019.

\bibitem{levine-aaai20}
A.~Levine and S.~Feizi.
\newblock Robustness certificates for sparse adversarial attacks by randomized
  ablation.
\newblock In {\em AAAI}, pages 4585--4593, 2020.

\bibitem{lpmms-cj16}
M.~H. Liffiton, A.~Previti, A.~Malik, and J.~Marques{-}Silva.
\newblock Fast, flexible {MUS} enumeration.
\newblock {\em Constraints An Int. J.}, 21(2):223--250, 2016.

\bibitem{barrett-fto21}
C.~Liu, T.~Arnon, C.~Lazarus, C.~A. Strong, C.~W. Barrett, and M.~J.
  Kochenderfer.
\newblock Algorithms for verifying deep neural networks.
\newblock {\em Found. Trends Optim.}, 4(3-4):244--404, 2021.

\bibitem{lorini-jlc23}
X.~Liu and E.~Lorini.
\newblock A unified logical framework for explanations in classifier systems.
\newblock {\em J. Log. Comput.}, 33(2):485--515, 2023.

\bibitem{kumar-iclr20}
J.~Lu and M.~P. Kumar.
\newblock Neural network branching for neural network verification.
\newblock In {\em ICLR}, 2020.

\bibitem{lundberg-nips17}
S.~M. Lundberg and S.~Lee.
\newblock A unified approach to interpreting model predictions.
\newblock In {\em NeurIPS}, pages 4765--4774, 2017.

\bibitem{madry-iclr18}
A.~Madry, A.~Makelov, L.~Schmidt, D.~Tsipras, and A.~Vladu.
\newblock Towards deep learning models resistant to adversarial attacks.
\newblock In {\em ICLR}, 2018.

\bibitem{msi-aaai22}
J.~Marques{-}Silva and A.~Ignatiev.
\newblock Delivering trustworthy {AI} through formal {XAI}.
\newblock In {\em AAAI}, pages 12342--12350, 2022.

\bibitem{msjb-cav13}
J.~Marques{-}Silva, M.~Janota, and A.~Belov.
\newblock Minimal sets over monotone predicates in boolean formulae.
\newblock In {\em CAV}, pages 592--607, 2013.

\bibitem{dong-corr22}
M.~H. Meng, G.~Bai, S.~G. Teo, Z.~Hou, Y.~Xiao, Y.~Lin, and J.~S. Dong.
\newblock Adversarial robustness of deep neural networks: {A} survey from a
  formal verification perspective.
\newblock {\em CoRR}, abs/2206.12227, 2022.

\bibitem{miller-pr56}
G.~A. Miller.
\newblock The magical number seven, plus or minus two: Some limits on our
  capacity for processing information.
\newblock {\em Psychological review}, 63(2):81--97, 1956.

\bibitem{miller-aij19}
T.~Miller.
\newblock Explanation in artificial intelligence: Insights from the social
  sciences.
\newblock {\em Artif. Intell.}, 267:1--38, 2019.

\bibitem{vechev-pldi21}
M.~Mirman, A.~H{\"{a}}gele, P.~Bielik, T.~Gehr, and M.~T. Vechev.
\newblock Robustness certification with generative models.
\newblock In {\em PLDI}, pages 1141--1154, 2021.

\bibitem{molnar-bk20}
C.~Molnar.
\newblock {\em Interpretable machine learning}.
\newblock 2020.

\bibitem{vechev-popl22}
M.~N. M{\"{u}}ller, G.~Makarchuk, G.~Singh, M.~P{\"{u}}schel, and M.~T. Vechev.
\newblock {PRIMA:} general and precise neural network certification via
  scalable convex hull approximations.
\newblock {\em Proc. {ACM} Program. Lang.}, 6({POPL}):1--33, 2022.

\bibitem{pulina-cav10}
L.~Pulina and A.~Tacchella.
\newblock An abstraction-refinement approach to verification of artificial
  neural networks.
\newblock In {\em CAV}, pages 243--257, 2010.

\bibitem{kohli-nips19a}
C.~Qin, J.~Martens, S.~Gowal, D.~Krishnan, K.~Dvijotham, A.~Fawzi, S.~De,
  R.~Stanforth, and P.~Kohli.
\newblock Adversarial robustness through local linearization.
\newblock In {\em NeurIPS}, pages 13824--13833, 2019.

\bibitem{reiter-aij87}
R.~Reiter.
\newblock A theory of diagnosis from first principles.
\newblock {\em Artif. Intell.}, 32(1):57--95, 1987.

\bibitem{guestrin-kdd16}
M.~T. Ribeiro, S.~Singh, and C.~Guestrin.
\newblock "why should {I} trust you?": Explaining the predictions of any
  classifier.
\newblock In {\em KDD}, pages 1135--1144, 2016.

\bibitem{guestrin-aaai18}
M.~T. Ribeiro, S.~Singh, and C.~Guestrin.
\newblock Anchors: High-precision model-agnostic explanations.
\newblock In {\em AAAI}, pages 1527--1535, 2018.

\bibitem{robinson-bk03}
D.~J. Robinson.
\newblock {\em An introduction to abstract algebra}.
\newblock Walter de Gruyter, 2003.

\bibitem{kolter-icml20}
E.~Rosenfeld, E.~Winston, P.~Ravikumar, and J.~Z. Kolter.
\newblock Certified robustness to label-flipping attacks via randomized
  smoothing.
\newblock In {\em ICML}, pages 8230--8241, 2020.

\bibitem{kwiatkowska-ijcai18}
W.~Ruan, X.~Huang, and M.~Kwiatkowska.
\newblock Reachability analysis of deep neural networks with provable
  guarantees.
\newblock In J.~Lang, editor, {\em IJCAI}, pages 2651--2659, 2018.

\bibitem{kwiatkowska-ijcai19}
W.~Ruan, M.~Wu, Y.~Sun, X.~Huang, D.~Kroening, and M.~Kwiatkowska.
\newblock Global robustness evaluation of deep neural networks with provable
  guarantees for the hamming distance.
\newblock In {\em IJCAI}, pages 5944--5952, 2019.

\bibitem{vechev-aaai21}
A.~Ruoss, M.~Baader, M.~Balunovic, and M.~T. Vechev.
\newblock Efficient certification of spatial robustness.
\newblock In {\em AAAI}, pages 2504--2513, 2021.

\bibitem{muller-ieee-proc21}
W.~Samek, G.~Montavon, S.~Lapuschkin, C.~J. Anders, and K.~M{\"{u}}ller.
\newblock Explaining deep neural networks and beyond: {A} review of methods and
  applications.
\newblock {\em Proc. {IEEE}}, 109(3):247--278, 2021.

\bibitem{seshia-cacm22}
S.~A. Seshia, D.~Sadigh, and S.~S. Sastry.
\newblock Toward verified artificial intelligence.
\newblock {\em Commun. {ACM}}, 65(7):46--55, 2022.

\bibitem{darwiche-ijcai18}
A.~Shih, A.~Choi, and A.~Darwiche.
\newblock A symbolic approach to explaining bayesian network classifiers.
\newblock In {\em IJCAI}, pages 5103--5111, 2018.

\bibitem{vechev-nips19b}
G.~Singh, R.~Ganvir, M.~P{\"{u}}schel, and M.~T. Vechev.
\newblock Beyond the single neuron convex barrier for neural network
  certification.
\newblock In {\em NeurIPS}, pages 15072--15083, 2019.

\bibitem{vechev-iclr19}
G.~Singh, T.~Gehr, M.~P{\"{u}}schel, and M.~T. Vechev.
\newblock Boosting robustness certification of neural networks.
\newblock In {\em ICLR}, 2019.

\bibitem{ermon-nips18}
Y.~Song, R.~Shu, N.~Kushman, and S.~Ermon.
\newblock Constructing unrestricted adversarial examples with generative
  models.
\newblock In {\em NeurIPS}, pages 8322--8333, 2018.

\bibitem{szegedy-iclr14}
C.~Szegedy, W.~Zaremba, I.~Sutskever, J.~Bruna, D.~Erhan, I.~J. Goodfellow, and
  R.~Fergus.
\newblock Intriguing properties of neural networks.
\newblock In {\em ICLR}, 2014.

\bibitem{tramer-icml22}
F.~Tram{\`{e}}r.
\newblock Detecting adversarial examples is (nearly) as hard as classifying
  them.
\newblock In {\em ICML}, pages 21692--21702, 2022.

\bibitem{valiant-cacm84}
L.~G. Valiant.
\newblock A theory of the learnable.
\newblock {\em Commun. {ACM}}, 27(11):1134--1142, 1984.

\bibitem{kutyniok-jair21}
S.~W{\"{a}}ldchen, J.~MacDonald, S.~Hauch, and G.~Kutyniok.
\newblock The computational complexity of understanding binary classifier
  decisions.
\newblock {\em J. Artif. Intell. Res.}, 70:351--387, 2021.

\bibitem{kwiatkowska-ijcai21}
B.~Wang, C.~Lyle, and M.~Kwiatkowska.
\newblock Provable guarantees on the robustness of decision rules to causal
  interventions.
\newblock In {\em IJCAI}, pages 4258--4265, 2021.

\bibitem{kumar-iclr19}
S.~Webb, T.~Rainforth, Y.~W. Teh, and M.~P. Kumar.
\newblock A statistical approach to assessing neural network robustness.
\newblock In {\em ICLR}, 2019.

\bibitem{kwiatkowska-tacas18}
M.~Wicker, X.~Huang, and M.~Kwiatkowska.
\newblock Feature-guided black-box safety testing of deep neural networks.
\newblock In {\em TACAS}, pages 408--426, 2018.

\bibitem{kwiatkowska-cvpr19}
M.~Wicker and M.~Kwiatkowska.
\newblock Robustness of 3d deep learning in an adversarial setting.
\newblock In {\em CVPR}, pages 11767--11775, 2019.

\bibitem{kwiatkowska-cvpr20}
M.~Wu and M.~Kwiatkowska.
\newblock Robustness guarantees for deep neural networks on videos.
\newblock In {\em CVPR}, pages 308--317, 2020.

\bibitem{kwiatkowska-tcs20}
M.~Wu, M.~Wicker, W.~Ruan, X.~Huang, and M.~Kwiatkowska.
\newblock A game-based approximate verification of deep neural networks with
  provable guarantees.
\newblock {\em Theor. Comput. Sci.}, 807:298--329, 2020.

\bibitem{barrett-corr22}
M.~Wu, H.~Wu, and C.~W. Barrett.
\newblock {VeriX}: Towards verified explainability of deep neural networks.
\newblock {\em CoRR}, abs/2212.01051, 2022.

\bibitem{kolter-icml22}
H.~Zhang, S.~Wang, K.~Xu, Y.~Wang, S.~Jana, C.~Hsieh, and J.~Z. Kolter.
\newblock A branch and bound framework for stronger adversarial attacks of relu
  networks.
\newblock In {\em ICML}, pages 26591--26604, 2022.

\end{thebibliography}
